\documentclass[11pt]{article}

\usepackage[final]{acl}

\usepackage{times}
\usepackage{latexsym}

\usepackage{inconsolata}

\usepackage[utf8]{inputenc} 
\usepackage[T1]{fontenc}    
\usepackage{amsfonts}       
\usepackage{nicefrac}       

\usepackage{comment}
\usepackage{microtype}
\usepackage{graphicx}
\usepackage{multirow} 
\usepackage{booktabs} 
\usepackage{adjustbox}
\usepackage{verbatim}
\usepackage{graphicx}
\usepackage{hyperref}
\usepackage{url}
\usepackage{makecell}
\usepackage{array}
\usepackage{wrapfig}
\usepackage{enumerate}
\usepackage{enumitem}
\usepackage{algorithm}
\usepackage{algpseudocode}
\usepackage{soul}
\usepackage{subcaption}
\usepackage[table]{xcolor}
\usepackage[figuresright]{rotating} 
\usepackage{tcolorbox}
\definecolor{myblue}{HTML}{F0F8FF} 
\definecolor{mypurple}{HTML}{E6E6FA} 
\definecolor{mygreen}{HTML}{F0FFF0} 
\definecolor{mypink}{HTML}{FFF0F5}

\newcommand{\att}{\mathsf{attn}}
\newcommand{\mha}{\mathsf{mha}}

\newcommand{\sm}{\mathsf{softmax}}

\usepackage{acronym}
\acrodef{mha}[MHA]{Multi-Head Attention}
\acrodef{ff}[FF]{Feed-Forward}
\acrodef{mle}[MLE]{Maximum Likelihood Estimate}

\usepackage{amsmath}
\usepackage{amssymb}
\usepackage{mathtools}
\usepackage{amsthm}

\DeclareMathAlphabet{\mathbsf}{OT1}{cmss}{bx}{n}
\DeclareMathAlphabet{\mathssf}{OT1}{cmss}{m}{sl}

\DeclareSymbolFont{bsfletters}{OT1}{cmss}{bx}{n}  
\DeclareSymbolFont{ssfletters}{OT1}{cmss}{m}{n}
\DeclareMathSymbol{\bsfGamma}{0}{bsfletters}{'000}
\DeclareMathSymbol{\ssfGamma}{0}{ssfletters}{'000}
\DeclareMathSymbol{\bsfDelta}{0}{bsfletters}{'001}
\DeclareMathSymbol{\ssfDelta}{0}{ssfletters}{'001}
\DeclareMathSymbol{\bsfTheta}{0}{bsfletters}{'002}
\DeclareMathSymbol{\ssfTheta}{0}{ssfletters}{'002}
\DeclareMathSymbol{\bsfLambda}{0}{bsfletters}{'003}
\DeclareMathSymbol{\ssfLambda}{0}{ssfletters}{'003}
\DeclareMathSymbol{\bsfXi}{0}{bsfletters}{'004}
\DeclareMathSymbol{\ssfXi}{0}{ssfletters}{'004}
\DeclareMathSymbol{\bsfPi}{0}{bsfletters}{'005}
\DeclareMathSymbol{\ssfPi}{0}{ssfletters}{'005}
\DeclareMathSymbol{\bsfSigma}{0}{bsfletters}{'006}
\DeclareMathSymbol{\ssfSigma}{0}{ssfletters}{'006}
\DeclareMathSymbol{\bsfUpsilon}{0}{bsfletters}{'007}
\DeclareMathSymbol{\ssfUpsilon}{0}{ssfletters}{'007}
\DeclareMathSymbol{\bsfPhi}{0}{bsfletters}{'010}
\DeclareMathSymbol{\ssfPhi}{0}{ssfletters}{'010}
\DeclareMathSymbol{\bsfPsi}{0}{bsfletters}{'011}
\DeclareMathSymbol{\ssfPsi}{0}{ssfletters}{'011}
\DeclareMathSymbol{\bsfOmega}{0}{bsfletters}{'012}
\DeclareMathSymbol{\ssfOmega}{0}{ssfletters}{'012}

\usepackage[capitalize,noabbrev]{cleveref}

\theoremstyle{plain}
\newtheorem{theorem}{Theorem}[section]
\newtheorem{proposition}[theorem]{Proposition}

\theoremstyle{definition}

\theoremstyle{remark}

\usepackage[textsize=tiny]{todonotes}

\title{\textsc{LongSpec}: Long-Context Lossless Speculative Decoding\\ with Efficient Drafting and Verification}

\renewcommand*{\thefootnote}{\fnsymbol{footnote}}
\author{
    \textbf{
    Penghui Yang\textsuperscript{2}\footnotemark[1],
    Cunxiao Du\textsuperscript{1}\footnotemark[1],
    Fengzhuo Zhang\textsuperscript{3},
    Haonan Wang\textsuperscript{3},} \\
    \textbf{Tianyu Pang\textsuperscript{1}, Chao Du\textsuperscript{1}, Bo An\textsuperscript{2}} \\[0.5ex]
    \textsuperscript{\rm 1} Sea AI Lab, 
    \textsuperscript{\rm 2} Nanyang Technological University,
    \textsuperscript{\rm 3} National University of Singapore\\
    \texttt{phyang.cs@gmail.com, cnsdunm@gmail.com, fzzhang@u.nus.edu}
}

\begin{document}
\maketitle
\footnotetext[1]{Equal contribution. Work done during Penghui Yang's associate membership at Sea AI Lab and when Cunxiao Du was at Sea AI Lab.}
\renewcommand*{\thefootnote}{\arabic{footnote}}

\begin{abstract}

As Large Language Models (LLMs) can now process extremely long contexts, efficient inference over these extended inputs has become increasingly important, especially for emerging applications like LLM agents that highly depend on this capability. Speculative decoding (SD) offers a promising lossless acceleration technique compared to lossy alternatives such as quantization and model cascades. However, most state-of-the-art SD methods are trained on short texts (typically fewer than 4k tokens), making them unsuitable for long-context scenarios. Specifically, adapting these methods to long contexts presents three key challenges: (1) the excessive memory demands posed by draft models due to large Key-Value (KV) cache; (2) performance degradation resulting from the mismatch between short-context training and long-context inference; and (3) inefficiencies in tree attention mechanisms when managing long token sequences. This work introduces \textsc{LongSpec}, a framework that addresses these challenges through three core innovations: a memory-efficient draft model with a constant-sized KV cache; novel position indices that mitigate the training–inference mismatch; and an attention aggregation strategy that combines fast prefix computation with standard tree attention to enable efficient decoding. Experimental results confirm the effectiveness of \textsc{LongSpec}, achieving up to a 3.26$\times$ speedup over strong \texttt{Flash Attention} baselines across five long-context understanding datasets, as well as a 2.34$\times$ reduction in wall-clock time on four math reasoning tasks with the QwQ model, demonstrating significant latency improvements for long-context applications.

\end{abstract}

\section{Introduction}
\label{sec:intro}

\begin{figure}[ht]
  \begin{center}
    \includegraphics[width=\linewidth]{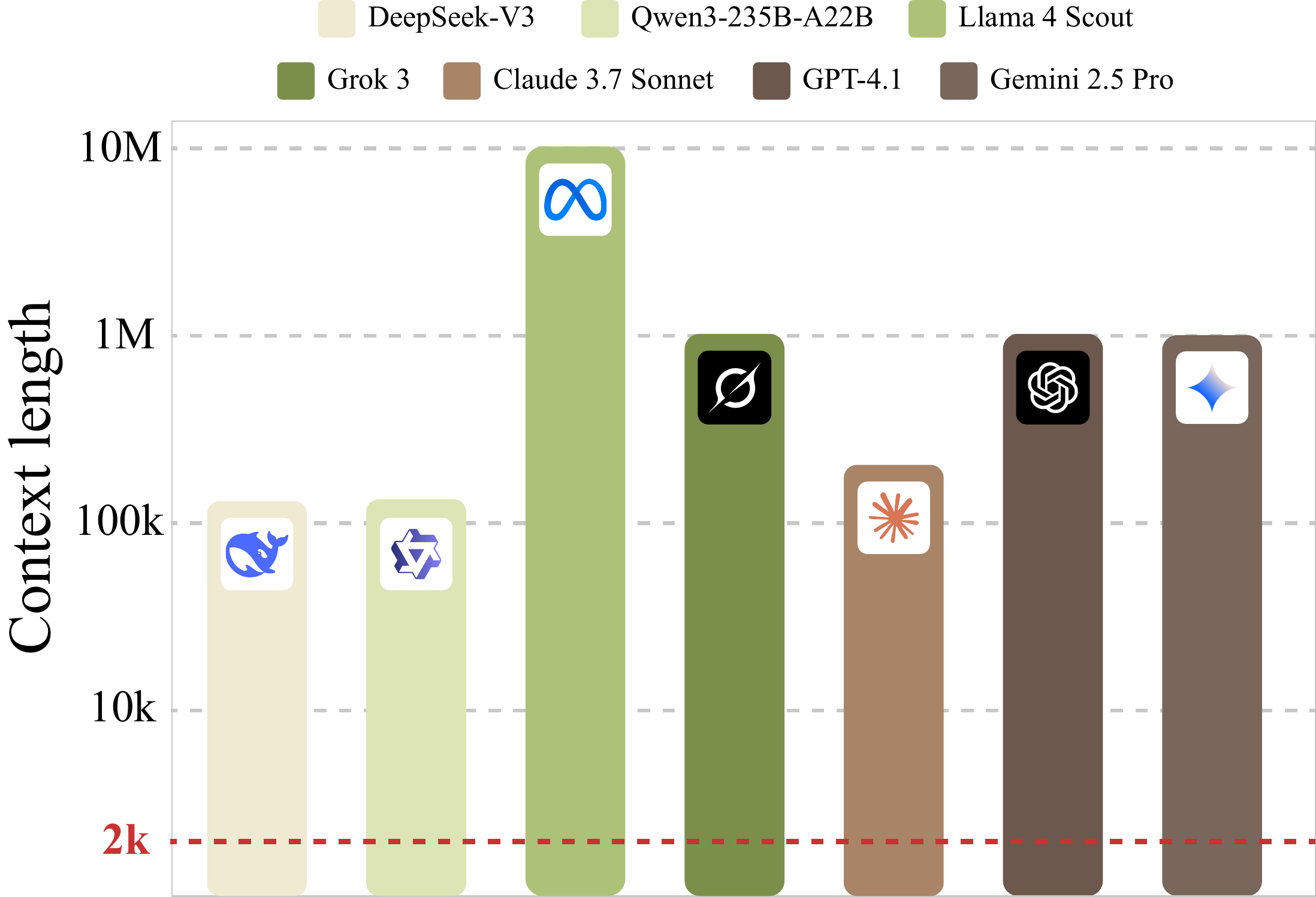}
  \end{center}
  \caption{The SoTA SD method, EAGLE, has a training context length of 2048, which is significantly shorter than the context lengths of modern LLMs.}
  \label{fig:teaser}
  \vspace{-.2cm}
\end{figure}


Large Language Models (LLMs) have demonstrated remarkable capabilities~\cite{achiam2023gpt}, and their ability to handle extensive contexts is becoming crucial for emerging applications such as LLM agents and long reasoning tasks~\cite{tan2025cradle, guo2025deepseek}, which now operate over context windows extending to millions of tokens~\cite{team2024gemini}. In these demanding long-context scenarios, the high inference latency of standard autoregressive decoding becomes a pronounced bottleneck. While various acceleration techniques such as quantization~\cite{lin2024awq}, sparse attention~\cite{li2024snapkv}, and model cascades~\cite{gupta2024language} have been proposed to mitigate this problem, they often compromise the output quality, rendering them lossy solutions. In contrast, speculative decoding (SD)~\cite{leviathan2023fast} offers a \textbf{lossless} acceleration strategy by using a smaller draft model to propose token sequences, which are then verified in parallel by the larger target model. However, state-of-the-art (SoTA) SD methods~\cite{li2024eagle}, which often rely on a small and standalone draft model, are mainly designed and evaluated on short-context data, typically with sequences shorter than 4k tokens (see Figure \ref{fig:teaser}). Although some existing SD methods can be extended to longer contexts, they often use the full target model with a compressed Key-Value (KV) cache as the draft model~\cite{chen2024magicdec, tiwari2025quantspec}. These approaches avoid the overhead of training a dedicated draft model, but their reliance on full target models, which are not sufficiently lightweight, limits the speed of draft generation. As a result, these methods may underperform compared to SoTA short-context SD techniques.
This divergence raises a critical question: 

\begin{center}
\textit{Why cannot SoTA SD methods for short contexts be directly applied to long sequences?}
\end{center}




In response to this question, we attribute the difficulty of directly adapting effective short-context SoTA SD techniques to long-context settings to three emergent challenges:
\begin{enumerate}[leftmargin=0.5cm,topsep=-1pt]

\item \textbf{Architecture:} In SoTA SD methods (\emph{e.g.}, EAGLE~\cite{li2024eagle}), the draft model’s KV cache still grows linearly with context length. This linear growth becomes a prohibitive memory bottleneck as the context length increases.

\item \textbf{Training:} Language model training typically relies on plentiful short-sequence data, while long-sequence data remains relatively scarce. The imbalance of training data makes it difficult for the model to generalize to longer contexts. To address this, conventional wisdom in training long-context LLMs employs length extrapolation, in particular by extending the Rotary Position Embedding (RoPE)~\cite{su2024roformer} base to accommodate longer contexts \cite{gao2024train, liu2024scaling, peng2024yarn}.
 However, this solution is not directly applicable to SoTA SD draft models, because their RoPE base must match that of the target model\footnote{SoTA SD techniques often require the draft model to utilize intermediate features (\emph{e.g.}, hidden states or KV cache) from the target model, which is crucial for providing richer information from the target model, enabling the draft model to better align with and predict the target model's outputs. See more explanations in Appendix~\ref{sec:defend}.}, which is fixed and already scaled for long-context scenarios.


\item \textbf{Inference:} The effectiveness of tree attention verification~\cite{miao2024specinfer, cai2024medusa} diminishes in long-context scenarios. In particular, common inference optimizations for long-context scenarios are primarily designed to handle regular, structured attention masks and are not optimized for arbitrary or unstructured attention masks. As a result, potential speedups from speculation may be lower than expected.

\end{enumerate}

To address these challenges, we introduce \textsc{LongSpec}, a comprehensive framework for efficient long-context lossless speculative decoding. \textsc{LongSpec} overcomes the aforementioned obstacles through three key innovations:
\begin{enumerate}[leftmargin=0.5cm,topsep=-1pt]
    \item \textbf{Memory-Efficient Architecture.} We propose a draft model architecture with constant memory usage regardless of context length, effectively resolving the scalability limitations of prior SoTA autoregressive draft models.
    \item \textbf{Effective Training Regimes.} We develop a novel training strategy involving Anchor-Offset Indices, enabling draft models trained on short sequences to robustly generalize to much longer contexts at inference time.
    \item \textbf{Fast Tree Attention.} We introduce Hybrid Tree Attention, a new computation method that significantly speeds up tree verification by decomposing attention calculations and leveraging optimized Triton kernels.
\end{enumerate}

Experiments on five long-context understanding datasets using five LLMs as target models show that our \textsc{LongSpec} can significantly reduce the long-context inference latency, achieving up to a 3.26$\times$ speedup over strong baselines with \texttt{Flash Attention}\footnote{In this paper, \texttt{Flash Attention} refers to the inference optimization technique \texttt{FlashDecoding}~\cite{dao2023flash}, implemented via the \texttt{flash\_attn\_with\_kvcache} function from the Flash Attention library~\cite{dao2022flash}.}, and up to a  7$\times$ speedup over common baselines using the HuggingFace implementation.
Additional experiments on four math reasoning datasets with the long reasoning model QwQ~\cite{qwen2024qwq} further validate the effectiveness of \textsc{LongSpec}, yielding a 2.34$\times$ speedup in wall-clock time. Furthermore, our proposed Anchor-Offset Indices enable models to reach the same loss level 3.93$\times$ faster, and our Hybrid Tree Attention reduces attention computation latency by approximately 75\% compared to the standard HuggingFace implementation.

\section{Related Work}
\label{sec:related}

Speculative decoding offers a promising approach to accelerating LLMs without compromising the quality of their outputs\footnote{While this paper focuses on original speculative decoding methods which are lossless, some recent works explore lossy speculative decoding (see Appendix \ref{sec:lossy_related} for a brief overview).}. Early efforts~\cite{xia2023speculative, leviathan2023fast, liu2024online, bae2023fast, liu2025pearl} rely on existing smaller LLMs to generate draft sequences. Some other methods aim to improve upon those early efforts~\cite{sun2023spectr, miao2024specinfer, chen2024cascade}. There are also some works using part of the target model as the draft model~\cite{liu2024kangaroo, zhang2024draft, elhoushi2024layerskip, xia2025swift}. Retrieval-based speculative decoding methods~\cite{fu2024break, he2024rest, zhao2024ouroboros, liu2025logitspec, shen2026double} offer an alternative by utilizing $N$-gram matching rather than relying on smaller models. These approaches bypass the need for additional model training, leveraging pre-existing data patterns to construct draft sequences efficiently. 

More recent advancements~\cite{cai2024medusa, li2024eagle, du2024glide, huang2025jakiro} have expanded on these foundations by designing specialized draft models and introducing tree speculation and verification techniques. These methods leverage customized draft models tailored for speculative decoding, achieving higher efficiency and performance. Additionally, the tree-based approaches employed in these methods allow for more adaptive and parallelizable decoding processes, paving the way for broader applications in real-world systems, including vision-language models~\citep{huang2025specvlm}.

Although speculative decoding has progressed significantly for conventional context lengths, only a few existing papers focus on lossless speculative decoding in long-context scenarios. TriForce~\cite{sun2024triforce} introduces a three-layer speculative decoding system that is scalable for long sequence generation. MagicDec~\cite{chen2024magicdec} uses speculative decoding to improve both the throughput and latency of LLM inference. QuantSpec~\cite{tiwari2025quantspec} employs a hierarchical 4-bit quantized KV
cache and 4-bit quantized weights for draft models. However, these methods mainly utilize the target model with the sparse KV cache as the draft model. The computation-intensive draft models restrict the practical usage of these methods when facing various batch sizes. In contrast, our work focuses on efficiently building a draft model with only one transformer block, achieving more effective performance across different scenarios.

\section{Methodology}

In this section, we present our framework \textsc{LongSpec} for Long-Context Speculative Decoding, which addresses three key challenges by
(1) designing a lightweight draft model architecture with constant-sized memory overhead,
(2) devising the training strategy with anchor-offset indices to handle long contexts effectively,
and (3) implementing a fast attention aggregation mechanism that leverages tree-based speculation and verification for practical usage.

\begin{figure*}[ht]
    \centering
    \includegraphics[width=0.9\linewidth]{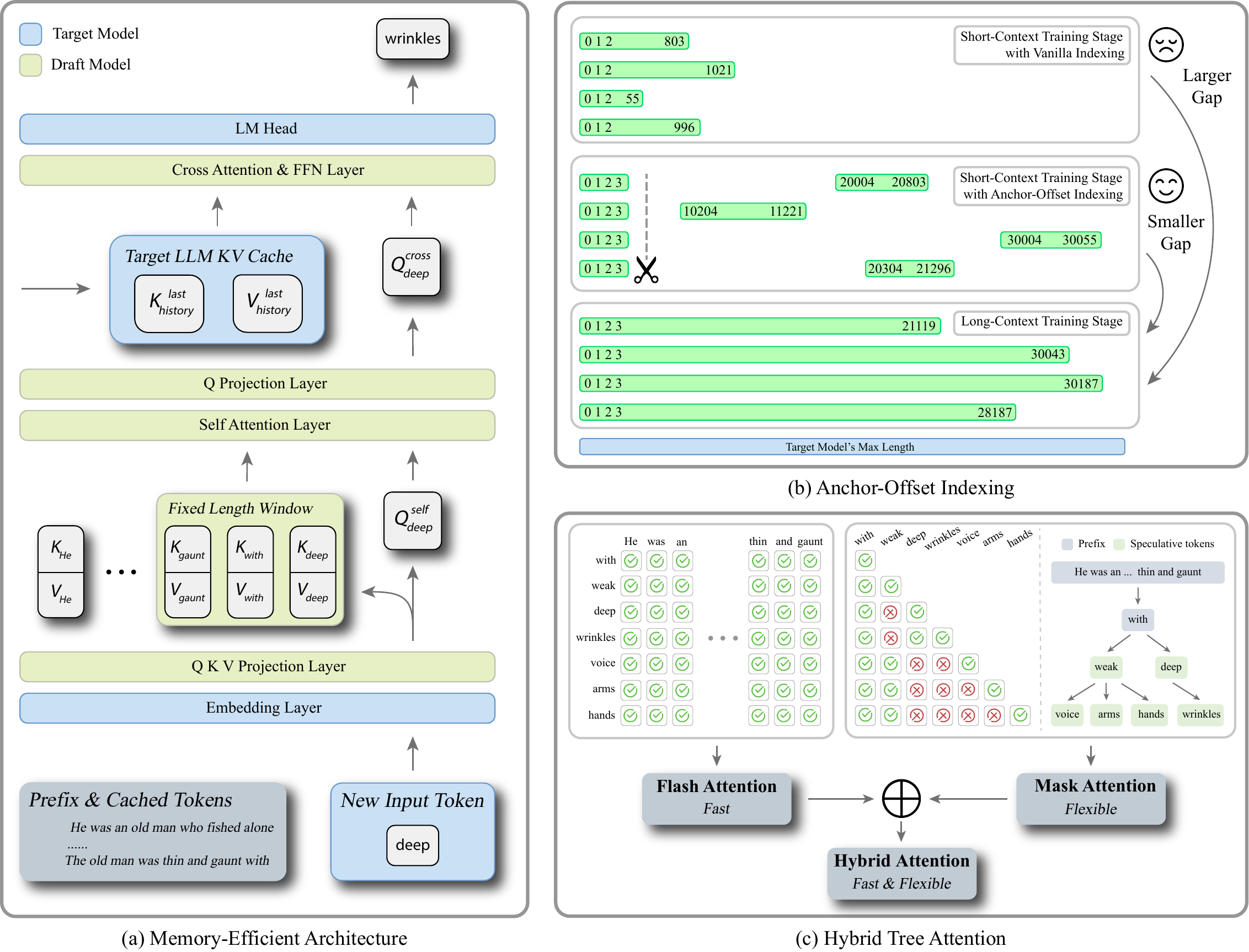}
    \caption{Illustration of the memory-efficient draft model, the Anchor-Offset Indices, and the Hybrid Tree Attention. (a) We use a sliding window self-attention layer to capture the local context information and a cross-attention layer to gather long-context information. (b) The differences between the vanilla indexing and the Anchor-Offset Indices. By introducing a randomly selected offset and some anchor indices, the Anchor-Offset Indices enable the short-context training stage to seamlessly integrate with the long-context training stage.
    (c) The Hybrid Tree Attention combines the advantages of \texttt{Flash Attention} and our Triton-implemented Attention.}
    \label{fig:main}
\end{figure*}

\subsection{Memory-Efficient Architecture}
\label{subsec: arch}

In previous work, the success of the SoTA model EAGLE depends on two key factors: (1) the hidden states provided by the target model, and (2) its autoregressive structure. However, an autoregressive draft model inevitably requires maintaining its own KV cache, which introduces additional overhead during long-context inference and demands substantial GPU memory, especially for tasks such as LLM agents and long reasoning that involve producing large amounts of output.

To avoid this extra memory overhead, we propose a draft model with constant memory usage independent of context length. As illustrated in Figure~\ref{fig:main}(a), our model consists of two components: the self-attention module and the following cross-attention module. The self-attention module focuses on modeling local context, while the cross-attention module captures long-range dependencies. To restrict memory usage, we apply a sliding-window attention mechanism to the self-attention module, a technique widely adopted in modern LLMs~\cite{beltagy2020longformer}. Hence, during inference, the self-attention does not exceed the window size, which we set to 512. 


For the cross-attention component, inspired by GliDe~\cite{du2024glide}, we leverage the KV cache of the target model (see Appendix~\ref{sec:defend} for a detailed explanation of how this benefits the draft model). This design not only enables better modeling of previous information but also completely removes additional storage overhead for long contexts, since the large model’s KV cache must be stored regardless of whether or not speculative decoding is employed.
Different from GliDe, we also share the weights of the Embedding Layer and LM Head between the target and draft models, which substantially reduces memory consumption for large-vocabulary LLMs such as LLaMA-3 (vocabulary size: 128,256)~\cite{dubey2024llama} and Qwen-2.5 (vocabulary size: 152,064)~\cite{yang2024qwen2}.

\subsection{Effective Training Regimes}
\label{subsec: Training}
\textbf{Anchor-Offset Indices.}
With vanilla position indices, which consist of successive integers starting from $0$, those indices appearing earlier in sequences occur more frequently than larger position indices~\cite{an2025does}, as shown in Figure~\ref{fig:main}(b) upper part.
Consequently, larger position indices receive insufficient training updates, which leads to a training-inference discrepancy.
As we point out in Section~\ref{sec:intro}, the common RoPE-based extrapolation cannot be directly used here because the RoPE base is fixed once the target model is chosen.
To leverage the target model’s KV cache, our draft model must keep the RoPE base the same as the target model.
To tackle this challenge, we can only leverage carefully designed indices. These indices must ensure that (1) the position indices in the draft model can be sufficiently trained using short-context data and (2) the indices would not cause the target model to exhibit out-of-distribution behavior because the target model shares the same indices as the draft model during training.

To satisfy these constraints, we propose the Anchor-Offset Indices strategy. Specifically, we reserve the first four positions $[0,1,2,3]$ as \textit{attention sink} tokens\footnote{According to \cite{xiao2024efficient}, LLM exhibits an \textit{attention sink} phenomenon when dealing with long texts, which means the attention weights primarily concentrate on the first four tokens and the recent tokens.}, then assign all subsequent tokens to large consecutive indices starting at a random offset (\emph{e.g.}, $[0,1,2,3,8192,8193,8194,\dots]$). By exploiting the \textit{attention sink} phenomenon, we believe that utilizing Anchor-Offset Indices can naturally lead the target model to exhibit in-distribution behavior. The anchor indices and random offset ensure that every position index can be sufficiently trained, addressing the limitation of the vanilla one that repeatedly trains only smaller indices. In our experiments, adopting these indices in the target model only increases the loss by approximately 0.001, indicating that the target model is indeed well-suited to such changes. Pseudo code can be found in Appendix~\ref{appendix:pseudo-code}.

\textbf{Flash Noisy Training.} During training, our draft model leverages the KV cache from a large model, while this KV cache is not always visible during inference. 
This is because the large model only updates its KV cache upon verification completion. Concretely, for the $t$-th cross-attention query \(Q_t\) in the draft model, we can only guarantee access to the corresponding key-value states \(K_{<t'}\), \(V_{<t'}\) satisfying
$ 1\le|t' - t|<\gamma$,
where \(\gamma\) is the number of speculative steps.

To ensure consistency between training and inference, a straightforward solution would be to add an attention mask~\cite{du2024glide}. However, this method is incompatible with \texttt{Flash Attention}, which would significantly degrade training speed and cause prohibitive memory overhead, particularly in long-context training scenarios. Therefore, we propose a technique called \textbf{flash noisy training}. During training, we randomly shift the indices of queries and key-value states with $1 \le j < \gamma$. 
Suppose the sequence length is $l$, then we compute
\[
    O_{\geq j} 
    = \att \bigl(Q_{\geq j}, \,K_{< l-j}, \,V_{< l-j}\bigr).
\]
In this way, we effectively simulate the same visibility constraints as in the inference phase, \emph{i.e.}, 
$ 1\le|t' - t|<\gamma$,
thereby aligning the behavior at training time with the inference behavior. When using Flash Noisy Training, we observe a 14.7\% increase in acceptance length compared to training without it, with improvements most concentrated on the final speculated tokens. This highlights its role in mitigating the training-inference gap. Pseudo code can be found in Appendix~\ref{appendix:pseudo-code}.

\subsection{Fast Tree Attention}

\label{subsec: tree attention}
Tree Speculative Decoding~\cite{miao2024specinfer} leverages speculation trees and the causal structure of LLMs so that a draft model can propose multiple candidate sequences, while the target model only needs to verify them once, without altering the final results. 
In this process, \emph{Tree Attention} plays a key role in ensuring both correctness and efficiency. Early works~\citep{cai2024medusa, li2024eagle} apply attention masks derived from prefix trees to the $QK^\mathsf{T}$ attention matrix, thus disabling wrong combinations between speculation tokens.
However, these methods only run on PyTorch’s eager execution mode, precluding more advanced attention kernels such as \texttt{Flash Attention}.
As a result, the inference speed decreases significantly when the sequence length increases.

To address these performance bottlenecks, we propose a \textbf{Hybrid Tree Attention} mechanism, as illustrated in Figure~\ref{fig:main}(c). Our method is based on two key insights: 1) when performing Tree Attention, as illustrated in the left part of Figure~\ref{fig:main}(c), the queries and the cached key-value pairs $\{K_{\mathrm{cache}}, V_{\mathrm{cache}}\}$ do not require additional masks; 2) only the queries and the key-value pairs $\{K_{\mathrm{specs}}, V_{\mathrm{specs}}\}$ from the current speculative tokens need masking as illustrated in the right part of Figure~\ref{fig:main}(c), and the number of such speculative tokens is typically small.
Based on these observations, we adopt a divide and aggregate approach that splits the attention computation into two parts and merges them afterward.

\textbf{Splitting Key-Value Pairs.}
We partition all key-value pairs into two groups: $\{K_{\mathrm{cache}}, V_{\mathrm{cache}}\}$: the cached part of the main sequence, which requires no attention mask;
and  $\{K_{\mathrm{specs}}, V_{\mathrm{specs}}\}$: the speculation-stage part, which needs attention masks. For $\{K_{\mathrm{cache}}, V_{\mathrm{cache}}\}$, we invoke the efficient \texttt{Flash Attention} kernel. For $\{K_{\mathrm{specs}}, V_{\mathrm{specs}}\}$, we use our custom Triton kernel \texttt{fused\_mask\_attn}, which applies blockwise loading and masking in the KV dimension, enabling fast computation of attention.
This step yields two sets of attention outputs $\{O_{\mathrm{cache}}, O_{\mathrm{specs}}\}$ along with their corresponding denominators (\emph{i.e.}, log-sum-exp of all attention scores) $\{\mathrm{LSE}_{\mathrm{cache}}, \mathrm{LSE}_{\mathrm{specs}}\}$.

\textbf{Aggregation.}
We then combine these two parts into the final attention output $O_{\mathrm{merge}}$ via a log-sum-exp trick. First, we compute
\vspace{-.1cm}
\[
\begin{aligned}
\mathrm{LSE}_{\mathrm{merge}}
&= \log\Bigl(\exp\bigl(\mathrm{LSE}_{\mathrm{cache}}\bigr) + \exp\bigl(\mathrm{LSE}_{\mathrm{specs}}\bigr)\Bigr),
\end{aligned}
\]
\vspace{-.1cm}
and then apply a weighted summation to the two outputs:
\vspace{-.1cm}
\[
\begin{aligned}
O_{\mathrm{merge}}
=\; &O_{\mathrm{cache}} \cdot \exp\bigl(\mathrm{LSE}_{\mathrm{cache}} - \mathrm{LSE}_{\mathrm{merge}}\bigr)\\ +\; &O_{\mathrm{specs}} \cdot \exp\bigl(\mathrm{LSE}_{\mathrm{specs}} - \mathrm{LSE}_{\mathrm{merge}}\bigr).
\end{aligned}
\]
\vspace{-.1cm}
The theoretical guarantee is provided in Appendix~\ref{appendix:attn_aggr}.
As outlined above, this hybrid approach employs the highly efficient \texttt{Flash Attention} kernel for most of the computations in long-sequence inference and only uses a custom masking attention \texttt{fused\_mask\_attn} for the small number of speculative tokens. 
The kernel \texttt{fused\_mask\_attn} follows the design philosophy of \texttt{Flash Attention 2}~\cite{dao2023flash} by splitting \(Q\), \(K_{\text{specs}}\), and \(V_{\text{specs}}\) into small blocks. This strategy reduces global memory I/O and fully leverages GPU streaming multiprocessors. Furthermore, for each block in the computation of \(QK_{\text{specs}}^\top\), the mask matrix is loaded and used to apply the masking operation.
The Hybrid Tree Attention effectively balances the parallel verification of multiple branches with improved inference speed, all without compromising correctness.

\section{Experiments}

\subsection{Settings}

\textbf{Target and draft models.} 
We select four widely-used long-context LLMs, Vicuna~(including 7B and 13B)~\cite{chiang2023vicuna}, LongChat~(including 7B and 13B)~\cite{li2023longchat}, LLaMA-3.1-8B-Instruct~\cite{dubey2024llama}, and QwQ-32B~\cite{qwen2024qwq}, as target models.
In order to make the draft model and target model more compatible, our draft model is consistent with the target model in various parameters, such as the number of KV heads. 

\textbf{Training Process.} 
We first train our draft model with Anchor-Offset Indices on the SlimPajama-6B pretraining dataset~\cite{cerebras2023slimpajama}. 
The random offset is set as a random integer from 0 to 15k for Vicuna models and LongChat-7B, and 0 to 30k for the other three models because they have longer maximum context length.
Then we train our model on a small subset of the Prolong-64k long-context dataset~\cite{gao2024train} in order to gain the ability to handle long texts. 
Finally, we finetune our model on a self-built long-context supervised-finetuning~(SFT) dataset to further improve the model performance.
The position index of the last two stages is the vanilla indexing policy because the training data is sufficiently long.
We apply flash noisy training during all three stages to mitigate the training and inference inconsistency and the extra overhead of flash noisy training is negligible.
More details on model training can be found in Appendix~\ref{appendix:training_details}.

\begin{table*}[t]
\centering
\caption{Average acceptance length $\tau$, decoding speed (tokens/s), and speedups across different models and settings. Specifically, ``Vanilla HF'' refers to HuggingFace’s PyTorch-based attention implementation, while ``Vanilla FA'' employs \texttt{Flash Attention}. The speedup statistic calculates the acceleration ratio relative to the Vanilla HF method. For the analysis of the reasons for the low speedup ratio of MagicDec, see Section~\ref{sec:main_results} and \ref{sec:throughput}. All results are computed at $T=0$.}
\label{tab:final_table}
\resizebox{\linewidth}{!}{
\begin{tabular}{c c c c c c c c c c c c c c c c c}
\toprule
 & \multirow{2}{*}{\textbf{Setting}} 
& \multicolumn{3}{c}{\textbf{GovReport}} 
& \multicolumn{3}{c}{\textbf{QMSum}} 
& \multicolumn{3}{c}{\textbf{Multi-News}} 
& \multicolumn{3}{c}{\textbf{LCC}} 
& \multicolumn{3}{c}{\textbf{RepoBench-P}} \\
\cmidrule(lr){3-5} \cmidrule(lr){6-8} \cmidrule(lr){9-11} \cmidrule(lr){12-14} \cmidrule(lr){15-17}
& 
& $\tau$ & Tokens/s & Speedup
& $\tau$ & Tokens/s & Speedup
& $\tau$ & Tokens/s & Speedup
& $\tau$ & Tokens/s & Speedup
& $\tau$ & Tokens/s & Speedup \\
\midrule

\multirow{4}{*}{\rotatebox{90}{V-7B}}
& Vanilla HF  
& 1.00 & 25.25 & -
& 1.00 & 18.12 & -
& 1.00 & 27.29 & -
& 1.00 & 25.25 & -
& 1.00 & 19.18 & - \\

& Vanilla FA  
& 1.00 & 45.76 & 1.00$\times$
& 1.00 & 43.68 & 1.00$\times$
& 1.00 & 55.99 & 1.00$\times$
& 1.00 & 54.07 & 1.00$\times$
& 1.00 & 46.61 & 1.00$\times$ \\

& MagicDec    
& 2.23 & 41.68 & 0.91$\times$
& 2.29 & 42.91 & 0.98$\times$
& 2.31 & 44.82 & 0.80$\times$
& 2.52 & 46.96 & 0.87$\times$
& 2.57 & 48.75 & 1.05$\times$ \\

& PLD
& 2.20 & 73.91 & 1.62$\times$
& 1.22 & 39.08 & 0.89$\times$
& 2.15 & 72.31 & 1.29$\times$
& 2.43 & 78.41 & 1.45$\times$
& 2.23 & 74.15 & 1.59$\times$ \\

& \textbf{LongSpec}  
& \textbf{3.57} & \textbf{102.23} & \textbf{2.23}$\times$
& \textbf{3.14} & \textbf{88.87}  & \textbf{2.04}$\times$
& \textbf{3.51} & \textbf{100.55} & \textbf{1.80}$\times$
& \textbf{3.73} & \textbf{107.30} & \textbf{1.99}$\times$
& \textbf{3.86} & \textbf{110.76} & \textbf{2.38}$\times$ \\
\midrule

\multirow{4}{*}{\rotatebox{90}{V-13B}}
& Vanilla HF  
& 1.00 & 17.25 & -
& 1.00 & 11.86 & -
& 1.00 & 18.81 & -
& 1.00 & 17.25 & -
& 1.00 & 13.44 & - \\

& Vanilla FA  
& 1.00 & 28.52 & 1.00$\times$
& 1.00 & 27.43 & 1.00$\times$
& 1.00 & 35.01 & 1.00$\times$
& 1.00 & 33.87 & 1.00$\times$
& 1.00 & 29.14 & 1.00$\times$ \\

& MagicDec    
& 2.95 & 38.24 & 1.34$\times$
& 2.87 & 37.15 & 1.35$\times$
& 2.97 & 39.47 & 1.13$\times$
& 2.96 & 38.40 & 1.13$\times$
& 2.94 & 36.66 & 1.26$\times$ \\

& PLD
& 1.37 & 32.10 & 1.13$\times$ 
& 1.28 & 28.29 & 1.03$\times$ 
& 1.35 & 34.97 & 1.00$\times$ 
& 1.34 & 36.24 & 1.07$\times$ 
& 1.32 & 30.60 & 1.05$\times$ \\

& \textbf{LongSpec}  
& \textbf{3.31} & \textbf{71.08} & \textbf{2.49}$\times$
& \textbf{2.76} & \textbf{57.15} & \textbf{2.08}$\times$
& \textbf{3.44} & \textbf{78.20} & \textbf{2.23}$\times$
& \textbf{3.57} & \textbf{81.00} & \textbf{2.39}$\times$
& \textbf{3.59} & \textbf{77.22} & \textbf{2.65}$\times$ \\
\midrule

\multirow{4}{*}{\rotatebox{90}{LC-7B}}
& Vanilla HF  
& 1.00 & 25.27 & -
& 1.00 & 14.11 & -
& 1.00 & 27.66 & -
& 1.00 & 25.27 & -
& 1.00 & 17.02 & - \\

& Vanilla FA  
& 1.00 & 42.14 & 1.00$\times$
& 1.00 & 36.87 & 1.00$\times$
& 1.00 & 50.19 & 1.00$\times$
& 1.00 & 54.17 & 1.00$\times$
& 1.00 & 42.69 & 1.00$\times$ \\

& MagicDec    
& 2.26 & 41.90 & 0.99$\times$
& 2.20 & 40.82 & 1.11$\times$
& 2.32 & 43.94 & 0.88$\times$
& 2.77 & 51.73 & 0.96$\times$
& 2.57 & 44.13 & 1.03$\times$ \\

& PLD
& 2.10 & 68.66 & 1.63$\times$ 
& 1.24 & 36.58 & 0.99$\times$ 
& 2.00 & 67.66 & 1.35$\times$ 
& 2.48 & 85.62 & 1.58$\times$ 
& 2.71 & 89.22 & 2.09$\times$ \\

& \textbf{LongSpec}  
& \textbf{3.59} & \textbf{101.43} & \textbf{2.41}$\times$
& \textbf{3.06} & \textbf{85.23} & \textbf{2.31}$\times$
& \textbf{3.41} & \textbf{97.93} & \textbf{1.95}$\times$
& \textbf{4.21} & \textbf{122.30} & \textbf{2.26}$\times$
& \textbf{4.03} & \textbf{115.27} & \textbf{2.70}$\times$ \\
\midrule

\multirow{4}{*}{\rotatebox{90}{LC-13B}}
& Vanilla HF  
& 1.00 & 17.72 & -
& 1.00 & 12.08 & -
& 1.00 & 18.74 & -
& 1.00 & 17.72 & -
& 1.00 & 13.85 & - \\

& Vanilla FA  
& 1.00 & 28.56 & 1.00$\times$
& 1.00 & 27.18 & 1.00$\times$
& 1.00 & 35.37 & 1.00$\times$
& 1.00 & 34.58 & 1.00$\times$
& 1.00 & 29.74 & 1.00$\times$ \\

& MagicDec    
& 2.40 & 31.37 & 1.10$\times$
& 2.38 & 30.84 & 1.13$\times$
& 2.43 & 32.58 & 0.92$\times$
& 2.68 & 35.77 & 1.03$\times$
& 2.85 & 35.67 & 1.20$\times$ \\

& PLD
& 1.67 & 35.35 & 1.24$\times$ 
& 1.18 & 24.10 & 0.89$\times$ 
& 1.85 & 43.74 & 1.24$\times$ 
& 1.88 & 49.12 & 1.42$\times$ 
& 1.80 & 41.07 & 1.38$\times$ \\

& \textbf{LongSpec}  
& \textbf{3.58} & \textbf{76.26} & \textbf{2.67}$\times$
& \textbf{3.15} & \textbf{64.41} & \textbf{2.37}$\times$
& \textbf{3.50} & \textbf{80.48} & \textbf{2.28}$\times$
& \textbf{4.01} & \textbf{90.92} & \textbf{2.63}$\times$
& \textbf{4.46} & \textbf{96.96} & \textbf{3.26}$\times$ \\
\midrule

\multirow{4}{*}{\rotatebox{90}{L-8B}}
& Vanilla HF  
& 1.00 & 21.59 & -
& 1.00 & 18.67 & -
& 1.00 & 29.91 & -
& 1.00 & 29.48 & -
& 1.00 & 22.77 & - \\

& Vanilla FA  
& 1.00 & 53.14 & 1.00$\times$
& 1.00 & 51.22 & 1.00$\times$
& 1.00 & 56.94 & 1.00$\times$
& 1.00 & 56.73 & 1.00$\times$
& 1.00 & 54.08 & 1.00$\times$ \\

& MagicDec    
& 2.04 & 36.14 & 0.68$\times$
& 2.00 & 35.78 & 0.70$\times$
& 2.33 & 39.57 & 0.70$\times$
& 2.65 & 46.95 & 0.83$\times$
& 2.61 & 44.39 & 0.82$\times$ \\

& PLD
& 2.08 & 77.45 & 1.46$\times$ 
& 1.52 & 45.76 & 0.89$\times$ 
& 1.94 & 78.00 & 1.37$\times$ 
& 1.59 & 54.75 & 0.97$\times$ 
& 1.38 & 45.70 & 0.85$\times$ \\

& \textbf{LongSpec}  
& \textbf{3.25} & \textbf{84.57} & \textbf{1.59}$\times$
& \textbf{2.99} & \textbf{75.68} & \textbf{1.48}$\times$
& \textbf{3.36} & \textbf{91.11} & \textbf{1.60}$\times$
& \textbf{3.28} & \textbf{89.33} & \textbf{1.57}$\times$
& \textbf{3.39} & \textbf{91.28} & \textbf{1.69}$\times$ \\

\bottomrule
\end{tabular}
}
\end{table*}

\textbf{Test Benchmarks.}
For conventional long-context understanding tasks, we select tasks from the LongBench benchmark~\cite{bai2024longbench} that involve generating longer outputs, because tasks with shorter outputs, such as document-QA, make it challenging to measure the speedup ratio fairly with speculative decoding. 
Specifically, we focus on long-document summarization and code completion tasks and conduct tests on five datasets: GovReport~\cite{huang2021efficient}, QMSum~\cite{zhong2021qmsum}, Multi-News~\cite{fabbri2019multi}, LCC~\cite{guo2023longcoder}, and RepoBench-P~\cite{liu2024repobench}. For math reasoning tasks, we test QwQ-32B on four math reasoning datasets: AIME24~\citep{li2024numinamath}, AMC~\citep{li2024numinamath}, MATH500~\citep{hendrycks2021measuring}, and Minerva Math~\citep{lewkowycz2022solving}.

We compare our method with the original target model, PLD~\cite{saxena2023prompt}, and MagicDec~\cite{chen2024magicdec}. PLD is the most popular retrieval-based method (also known as $n$-gram SD in vLLM~\cite{kwon2023efficient}), and MagicDec is a simple prototype of TriForce. 
To highlight the significance of \texttt{Flash Attention} in long-context scenarios, we also present the performance of the original target model using both eager attention implemented by HuggingFace and \texttt{Flash Attention} for comparison.
To make a fair comparison, we also use \texttt{Flash Attention} for baseline MagicDec.
The most important metric for speculative decoding is the \emph{walltime speedup ratio}, which is the actual test speedup ratio relative to vanilla autoregressive decoding. 
We also test the \emph{average acceptance length} $\tau$, \emph{i.e.}, the average number of tokens accepted per forward pass of the target LLM. 

\begin{figure*}[t]
    \centering
    \includegraphics[width=1\linewidth]{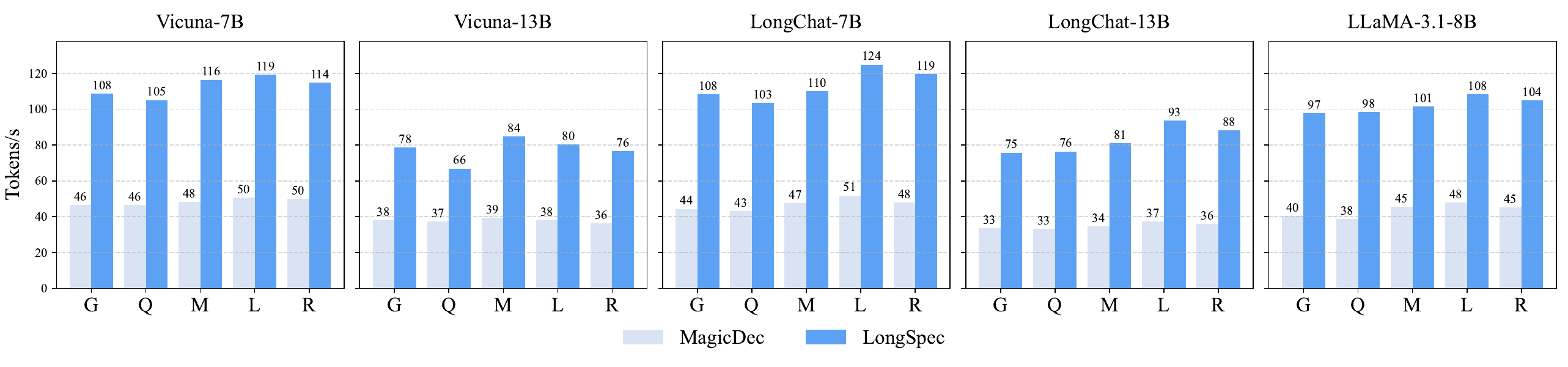}
    \caption{Decoding speed (tokens/s) across different models and settings. All results are computed at $T=1$. The letters G, Q, M, L, and R on the horizontal axis represent the datasets GovReport, QMSum, Multi-News, LCC, and RepoBench-P respectively.}
    \label{fig:final_fig}
    \vspace{-.5cm}
\end{figure*}

\subsection{Main Results}
\label{sec:main_results}

Table~\ref{tab:final_table} and Figure~\ref{fig:final_fig} show the decoding speeds and average acceptance lengths across the five evaluated datasets at $T=0$ and $T=1$, where $T$ denotes the temperature used in LLM sampling.
Our proposed method significantly outperforms all other approaches on both summarization tasks and code completion tasks. When $T=0$, on summarization tasks, our method can achieve an average acceptance length of around 3.5 and a speedup of up to 2.67$\times$; and on code completion tasks, our method can achieve an average acceptance length of around 4 and a speedup of up to 3.26$\times$. This highlights the robustness and generalizability of our speculative decoding approach, particularly in long-text generation tasks. At $T=1$, our method achieves around 2.5$\times$ speedup, maintaining a substantial lead over MagicDec. This indicates that our approach is robust across different temperature settings, further validating its soundness and efficiency.

Although PLD can accelerate generation on many datasets, it still does not match the performance of our proposed LongSpec. In some scenarios (\emph{e.g.}, when retrieval is minimal), PLD can even result in negative acceleration. For another baseline, MagicDec, while it demonstrates competitive acceptance rates compared to LongSpec, its speedup is noticeably lower in our experiments. This is because MagicDec is primarily designed for scenarios with large batch sizes and tensor parallelism. 
In low-batch-size settings, its draft model, which leverages all parameters of the target model with a sparse KV cache, becomes excessively heavy. 
This design choice leads to inefficiencies, as the draft model's computational overhead outweighs its speculative benefits. 
Our results reveal that MagicDec only achieves acceleration ratios~$>\!1$ on partial datasets when using a guess length $\gamma \!=\! 2$ and consistently exhibits negative acceleration around 0.7$\times$ when $\gamma\!\geq\!3$, further underscoring the limitations of this method in such configurations. The performance of MagicDec in larger batch sizes can be found in Section~\ref{sec:throughput}.

Lastly, we find that attention implementation plays a critical role in long-context speculative decoding performance. In our experiments, ``Vanilla HF'' refers to HuggingFace’s attention implementation, while ``Vanilla FA'' employs \texttt{Flash Attention}. The latter demonstrates nearly a $2\times$ speedup over the former, even as a standalone component, and our method can achieve up to $6\times$ speedup over HF Attention on code completion datasets. 
This result underscores the necessity for speculative decoding methods to be compatible with optimized attention mechanisms like \texttt{Flash Attention}, especially in long-text settings. Our hybrid tree attention approach achieves this compatibility, allowing us to fully leverage the advantages of \texttt{Flash Attention} and further speedup.

\subsection{Ablation Studies}

\begin{table}[t]
    \centering
    \caption{Performance comparison with and without Anchor-Offset Indices on the Multi-News and RepoBench-P datasets. Models with Anchor-Offset Indices achieve higher output speed and larger acceptance length, highlighting their efficiency and effectiveness.}
    \label{tab:ablation_1}
    \scalebox{0.75}{
    \begin{tabular}{c c c c c}
        \toprule
        & \multicolumn{2}{c}{\textbf{Multi-News}} 
        & \multicolumn{2}{c}{\textbf{RepoBench-P}} \\
        \cmidrule(lr){2-3} \cmidrule(lr){4-5}
        & $\tau$ & Tokens/s
        & $\tau$ & Tokens/s \\
        \midrule
        w/o Anchor-Offset & 3.20 & 85.98 & 3.26 & 85.21\\
        w/ Anchor-Offset  & 3.36 & 91.11 & 3.39 & 91.28\\
        \bottomrule
    \end{tabular}
    }
\end{table}

\textbf{Anchor-Offset Indices.} 
The experimental results demonstrate the significant benefits of incorporating the Anchor-Offset Indices. Figure~\ref{fig:ablation_1} shows that the model trained with Anchor-Offset Indices achieves a lower initial loss and final loss compared to the one trained without them when training on the real long-context dataset. 
Notably, the initialization with Anchor-Offset Indices reaches the same loss level $3.93\times$ faster than its counterpart. 
Table~\ref{tab:ablation_1} further highlights the performance improvements across two datasets, a summary dataset Multi-News, and a code completion dataset RepoBench-P. 
Models with Anchor-Offset Indices exhibit faster output speed and larger average acceptance length $\tau$. These results underscore the effectiveness of Anchor-Offset Indices in enhancing both training efficiency and model performance.

\begin{figure}
  \begin{center}
    \includegraphics[width=0.85\linewidth]{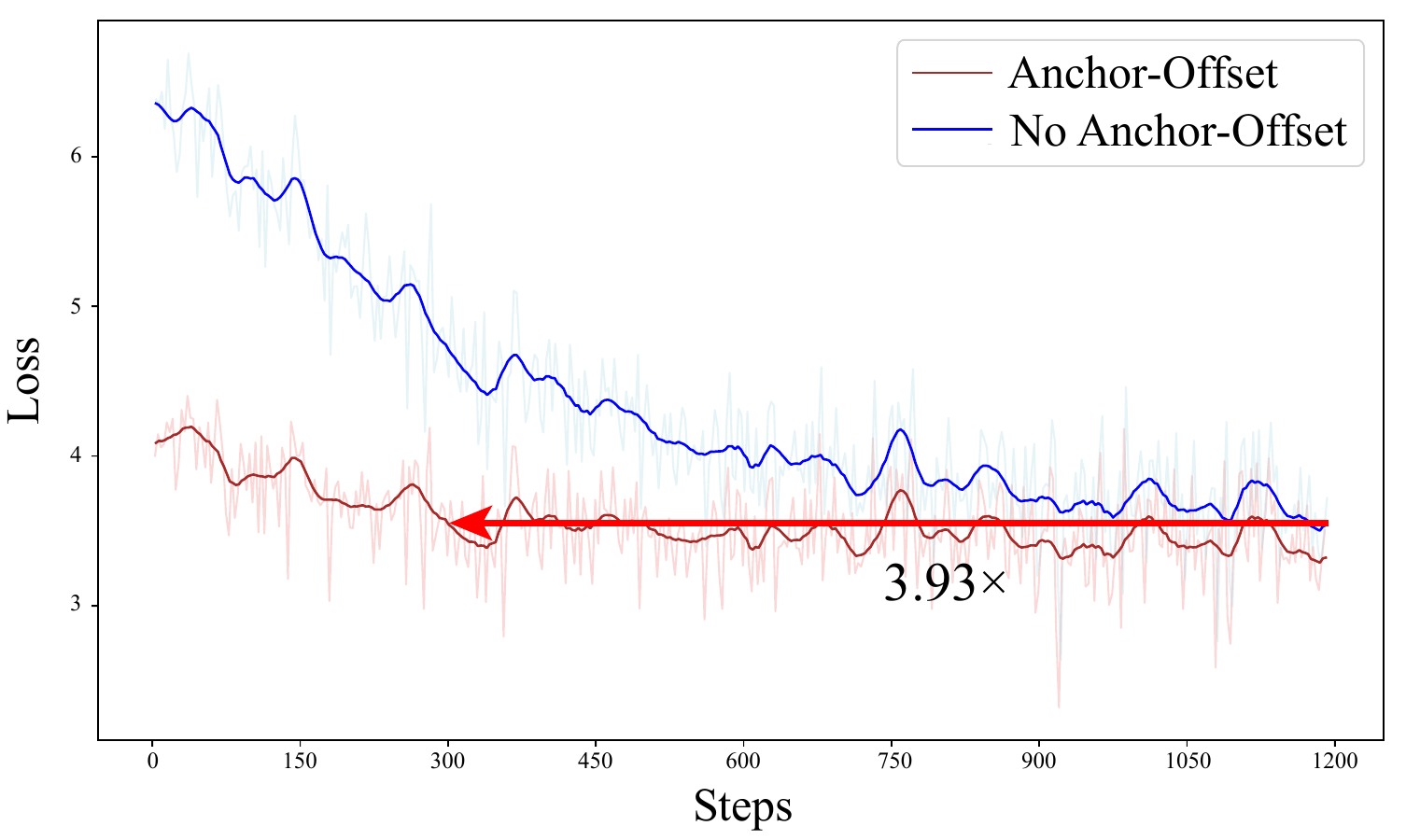}
  \end{center}
  \caption{Training loss curves on long-context data. Pretrained models with Anchor-Offset Indices exhibit lower initial and final loss, and reach the same loss level 3.93$\times$ faster compared to models without Anchor-Offset Indices.}
  \label{fig:ablation_1}
\end{figure}

\textbf{Hybrid Tree Attention.}
The results presented in Figure \ref{fig:ablation_2} highlight the effectiveness of the proposed Hybrid Tree Attention, which combines \texttt{Flash Attention} with the Triton kernel \texttt{fused\_mask\_attn}. 
While the time spent on the draft model forward pass and the target model FFN computations remain comparable across the two methods, the hybrid approach exhibits a significant reduction in latency for the target model's attention layer (the yellow part). 
Specifically, the attention computation latency decreases from 49.92 ms in the HF implementation to 12.54 ms in the hybrid approach, resulting in an approximately 75\% improvement. 
The verification step time difference is minimal, further solidifying the conclusion that the primary performance gains stem from optimizing the attention mechanism.

\begin{figure}[t]
    \centering
    \includegraphics[width=\linewidth]{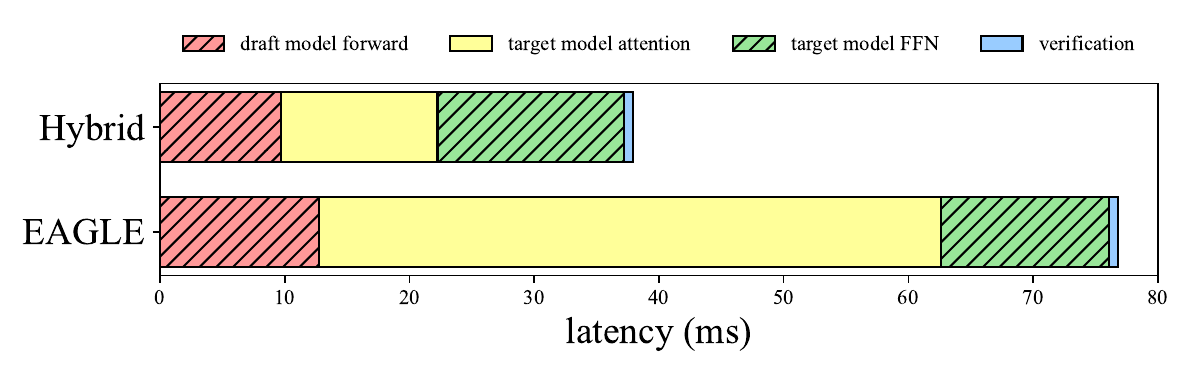}
    \caption{Latency breakdown for a single speculative decoding loop comparing the EAGLE implementation and the proposed Hybrid Tree Attention. Significant latency reduction is observed in the target model's attention layer (the yellow part) using our approach.}
    \label{fig:ablation_2}
\end{figure}

\begin{table}
\caption{Performance of our method on the QwQ-32B model on four math reasoning datasets, using a maximum output length of 32k tokens. The table shows the tokens generated per second and the mean number of accepted tokens $\tau$, where our approach achieves about 2.34$\times$ higher speed compared to the baseline on average and an average of 3.81 acceptance tokens.}
\label{tab:qwq}
\centering
\resizebox{\linewidth}{!}{
\begin{tabular}{c c c c c}
\toprule
Dataset & Metric & Vanilla & LongSpec & Improvement\\

\midrule
\multirow{2}{*}{AIME24} & $\tau$ & 1.00 & 3.82 & 3.82$\times$\\
\cmidrule(lr){2-5}
& Tokens/s & 18.92 & 42.63 & 2.25$\times$\\

\midrule
\multirow{2}{*}{AMC} & $\tau$ & 1.00 & 3.81 & 3.81$\times$\\
\cmidrule(lr){2-5}
& Tokens/s & 19.41 & 45.16 & 2.33$\times$\\

\midrule
\multirow{2}{*}{Minerva} & $\tau$ & 1.00 & 3.65 & 3.65$\times$\\
\cmidrule(lr){2-5}
& Tokens/s & 19.46 & 44.51 & 2.29$\times$\\

\midrule
\multirow{2}{*}{MATH500} & $\tau$ & 1.00 & 3.95 & 3.95$\times$\\
\cmidrule(lr){2-5}
& Tokens/s & 19.59 & 48.36 & 2.47$\times$\\

\bottomrule
\end{tabular}
}
\end{table}

\subsection{Long Reasoning Acceleration}

Long reasoning tasks have gained significant attention recently due to their ability to enable models to perform complex reasoning and problem-solving over extended outputs~\cite{qwen2024qwq, openai2024o1}. In these tasks, while the prefix input is often relatively short, the generated output can be extremely long, posing unique challenges in terms of efficiency and token acceptance. Our method is particularly well-suited for addressing these challenges, effectively handling scenarios with long outputs. It is worth mentioning that MagicDec is not suitable for such long-output scenarios because the initial inference stage of the long reasoning task is not the same as the traditional long-context task. In long reasoning tasks, where the prefix is relatively short, the draft model in MagicDec will completely degrade into the target model, failing to achieve acceleration.

We evaluate our method on the QwQ-32B model using four widely used benchmarks with a maximum output length set to 32k tokens. 
The results, illustrated in Table~\ref{tab:qwq}, demonstrate a significant improvement in both generation speed and average acceptance tokens. 
Specifically, our method achieves a generation rate of around 45 tokens/s, 2.34$\times$ higher than the strong \texttt{Flash Attention} baseline, and an average of 3.81 average acceptance tokens.
Notably, QwQ-32B with \textsc{LongSpec} achieves even lower latency than the standard 7B model with \texttt{Flash Attention}, demonstrating that our method effectively accelerates the long reasoning model.
These findings not only highlight the effectiveness of our method in the long reasoning task but also provide new insights into lossless inference acceleration for the o1-like model.
We believe speculative decoding will play a crucial role in accelerating this type of model in the future.

\begin{figure}
  \begin{center}
    \includegraphics[width=0.9\linewidth]{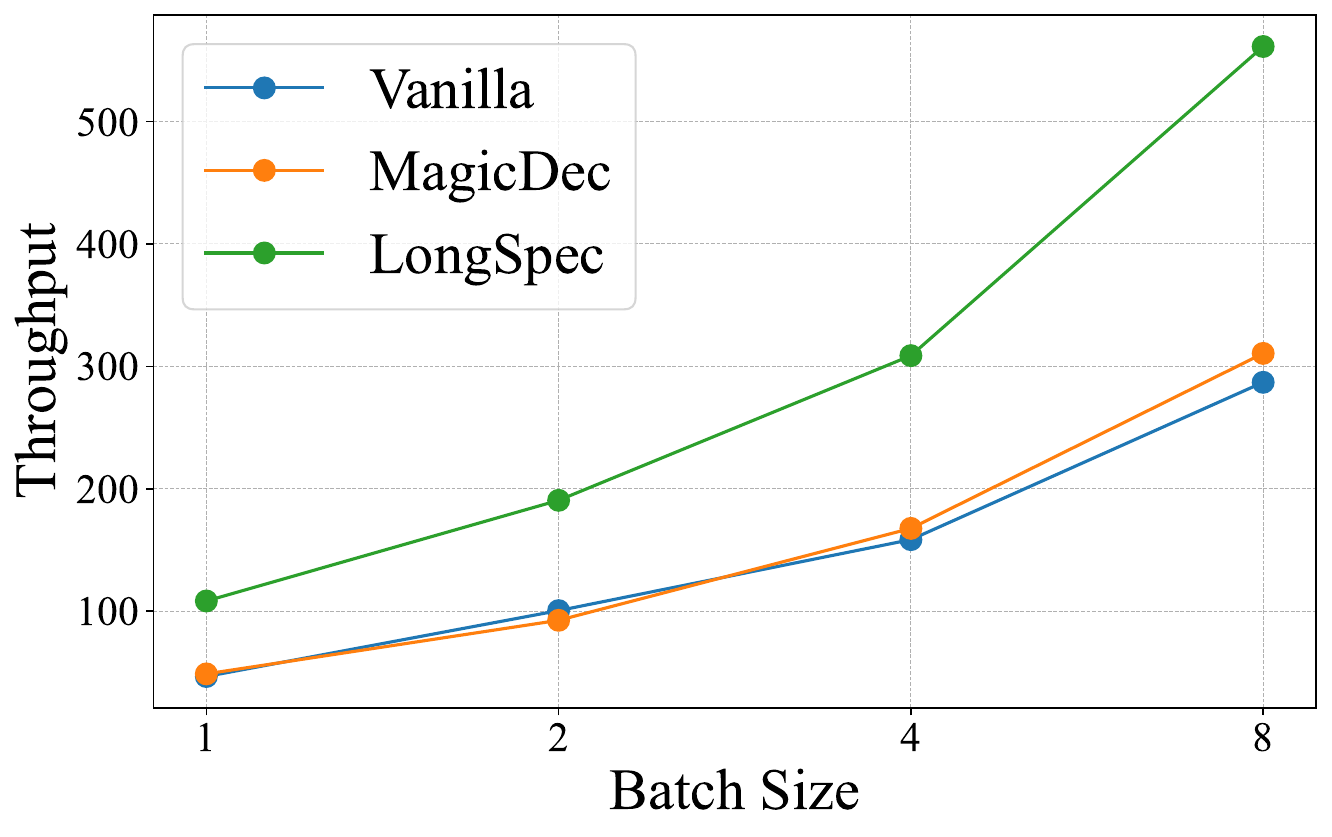}
  \end{center}
  \caption{Throughput comparison of Vanilla, MagicDec, and \textsc{LongSpec}.}
  \label{fig:throughput}
  \vspace{-.5cm}
\end{figure}

\vspace{-.2cm}
\subsection{Throughput}
\label{sec:throughput}

As illustrated in Figure \ref{fig:throughput}, the throughput results of Vicuna-7B on the RepoBench-P dataset show that \textsc{LongSpec} consistently outperforms both Vanilla and MagicDec across all batch sizes. At a batch size of 8, \textsc{LongSpec} achieves a throughput of 561.32 tokens/s, approximately 1.8$\times$ higher than MagicDec (310.58 tokens/s) and nearly 2$\times$ higher than Vanilla (286.96 tokens/s). MagicDec, designed with throughput optimization in mind, surpasses Vanilla as the batch size increases, reflecting its targeted improvements. However, \textsc{LongSpec} still sustains its advantage, maintaining superior throughput across all tested batch sizes.

\section{Conclusion}

In this paper, we propose \textsc{LongSpec}, a novel framework designed to enhance lossless speculative decoding for long-context scenarios. Unlike previous speculative decoding methods that primarily focus on short-context settings, \textsc{LongSpec} directly addresses three key challenges: excessive memory overhead, inadequate training for large position indices, and inefficient tree attention computation. To mitigate memory constraints, we introduce an efficient draft model architecture that maintains a constant memory footprint by leveraging a combination of sliding window self-attention and cache-free cross-attention. To resolve the training limitations associated with short-context data, we propose the Anchor-Offset Indices, ensuring that large positional indices are sufficiently trained even within short-sequence datasets. Finally, we introduce Hybrid Tree Attention, 
which efficiently integrates tree-based speculative decoding with \texttt{Flash Attention}. Extensive experiments demonstrate the effectiveness of \textsc{LongSpec} in long-context understanding tasks and real-world long reasoning tasks. Our findings highlight the importance of designing speculative decoding methods specifically tailored for long-context settings and point to promising directions for future research in efficient large-scale language model inference.

\bibliography{example_paper}




\newpage
\appendix

\section{Related Work about Lossy Speculative Decoding}
\label{sec:lossy_related}

While original speculative decoding methods are mainly lossless, some recent works try to relax the constraints and explore lossy speculative decoding. For instance, BiLD~\cite{kim2023speculative} employs a small model for autoregressive text generation, with a larger model occasionally invoked non-autoregressively to refine inaccurate predictions, thereby achieving speedups with minimal quality degradation. Narasimhan et al.~\cite{narasimhan2025faster} introduce speculative cascading, a method that integrates cascade-style deferral rules with speculative execution to yield better cost-quality trade-offs than either approach alone. Another approach, MTAD~\cite{qin2025multi}, uses a smaller auxiliary model to approximate the multi-token joint distribution of a larger model, enhancing both inference speed and output effectiveness by accepting a bounded error in this approximation. To address the rejection of high-quality but non-aligned draft tokens, Bachmann et al.~\cite{bachmann2025judge} propose adapting the verification step by training a compact ``judge'' module to recognize valid continuations even without perfect target model alignment, significantly boosting acceptance rates and speed. RSD~\cite{liao2025reward} incorporates a process reward model to evaluate intermediate decoding steps, dynamically deciding target model invocation and introducing a controlled bias towards high-reward outputs to optimize the cost-quality trade-off. RAPID~\cite{chen2025long} employs a RAG-based approach on shortened contexts as its drafter. TokenSwift~\cite{wu2025from} comprehensively uses LLMs with partial KV cache and $N$-gram tables to accelerate ultra-long sequence generation (up to 100k tokens) while reducing computation time from hours to minutes.

\section{The Intuition of Why KV Cache Can Help}
\label{sec:defend}

The KV cache stores the contextual information the model accumulates while processing previous tokens. When predicting the next token, the target model relies on three components: the KV cache (contextual memory), input word embeddings, and model parameters.

In our method, the draft model already shares the input embeddings with the target model, so the primary differences in their predictions stem from the KV cache and internal parameters. By allowing the draft model to use the KV cache generated by the target model, we eliminate another source of variation. As a result, the only remaining difference between their predictions comes from the model parameters. This sharing aligns the draft model’s predictions more closely with those of the target model, as it removes discrepancies caused by differing contextual representations.

\section{Correctness for Attention Aggregation}
\label{appendix:attn_aggr}

Because the query matrix $Q$ can be decomposed into several rows, each representing a separate query $q$, we can only consider the output of each row's $q$ after calculating attention with KV. In this way, we can assume that the KV involved in the calculation has undergone the tree mask, which can simplify our proof. We only need to prove that the output $o$ obtained from each individual $q$ meets the requirements, which can indicate that the overall output $O$ of the entire matrix $Q$ also meets the requirements.

\begin{proposition}
    Denote the log-sum-exp of the merged attention as follows:
    \begin{equation*}
        \mathrm{LSE}_{\mathrm{merge}} = \log\Bigl(\exp\bigl(\mathrm{LSE}_{\mathrm{cache}}\bigr) \;+\; \exp\bigl(\mathrm{LSE}_{\mathrm{specs}}\bigr)\Bigr),
    \end{equation*}
    Then we can write the merged attention output in the following way:
    \begin{align*}
    o_{\mathrm{merge}} = &o_{\mathrm{cache}} \cdot \exp\bigl(\mathrm{LSE}_{\mathrm{cache}} - \mathrm{LSE}_{\mathrm{merge}}\bigr) \\+& o_{\mathrm{specs}} \cdot\exp\bigl(\mathrm{LSE}_{\mathrm{specs}} - \mathrm{LSE}_{\mathrm{merge}}\bigr).
    \end{align*}
\end{proposition}

\begin{proof}
    
A standard scaled dot-product attention for $ q $ (of size $d_{qk}$) attending to $ K_{\mathrm{merge}} $ and $ V_{\mathrm{merge}} $ (together of size $(M+N) \times d_{qk}$ and $(M+N) \times d_v$ respectively) can be written as:

\begin{align*}
    o_{\mathrm{merge}} &= \mha\left(q, K_{\mathrm{merge}}, V_{\mathrm{merge}}\right) \\&=
    \sm\left(
      qK_{\mathrm{merge}}^\top/\sqrt{d_{qk}}
    \right) V_{\mathrm{merge}}.
\end{align*}

Because $ K $ and $ V $ are formed by stacking $\left(K_{\mathrm{specs}}, K_{\mathrm{cache}}\right)$ and $\left(V_{\mathrm{specs}}, V_{\mathrm{cache}}\right)$, we split the logit matrix accordingly:

\begin{align*}
    q K_{\mathrm{merge}}^\top / \sqrt{d_{qk}} = 
    \texttt{concat}\Bigl(&
    \underbrace{
      q \, K_{\mathrm{cache}}^\top / \sqrt{d_{qk}}
    }_{\mathrm{sub-logits\ for\ history}}
    \;, \\&
    \underbrace{
      q \, K_{\mathrm{specs}}^\top / \sqrt{d_{qk}}
    }_{\mathrm{sub-logits\ for\ new}}
    \Bigr).
\end{align*}

Denote these sub-logit matrices as:
\begin{align*}
Z_{\mathrm{cache}} \;&=\; q \,K_{\mathrm{cache}}^\top / \sqrt{d_{qk}}
,\;\\
Z_{\mathrm{specs}} \;&=\; q \, K_{\mathrm{specs}}^\top / \sqrt{d_{qk}}.
\end{align*}

Each row $i$ of $Z_{\mathrm{specs}}$ corresponds to the dot products between the $i$-th query in $q$ and all rows in $K_{\mathrm{specs}}$, while rows of $Z_{\mathrm{cache}}$ correspond to the same query but with $K_{\mathrm{cache}}$.

In order to combine partial attentions, we keep track of the log of the sum of exponentials of each sub-logit set. Concretely, define:

\begin{align}
    \mathrm{LSE}_{\mathrm{cache}} &= \log\left(\sum\nolimits_{j=1}^{N} \exp\left(Z_{\mathrm{cache}}^{(j)}\right)\right),\nonumber\\
    \;
    \mathrm{LSE}_{\mathrm{specs}} &= \log\left(\sum\nolimits_{j=1}^{M} \exp\left(Z_{\mathrm{specs}}^{(j)}\right)\right),
    \,
\end{align}

where $Z_{\mathrm{specs}}^{(j)}$ denotes the logit for the $j$-th element, and similarly for $Z_{\mathrm{cache}}^{(j)}$.





Then $o_{\mathrm{cache}}$ and $o_{\mathrm{specs}}$ can be written as:
\begin{align}
    \label{equ:split_attn}
    o_{\mathrm{cache}} &= \frac{\sum_{j=1}^{N} \exp\left(Z_{\mathrm{cache}}^{(j)}\right) V_{\mathrm{cache}}^{(j)}}{\exp\left(\mathrm{LSE}_{\mathrm{cache}}\right)}, \nonumber\\
    o_{\mathrm{specs}} &= \frac{\sum_{j=1}^{M} \exp\left(Z_{\mathrm{specs}}^{(j)}\right) V_{\mathrm{specs}}^{(j)}}{\exp\left(\mathrm{LSE}_{\mathrm{specs}}\right)}.
\end{align}

And the whole attention score can be written as:


\begin{align}
N_{\mathrm{num}} &=
\sum_{j=1}^{N} \exp\bigl(Z_{\mathrm{cache}}^{(j)}\bigr) V_{\mathrm{cache}}^{(j)}\nonumber\\
&+ \sum_{j=1}^{M} \exp\bigl(Z_{\mathrm{specs}}^{(j)}\bigr) V_{\mathrm{specs}}^{(j)}, \nonumber\\[0.5em]
D_{\mathrm{den}} &=
\exp\bigl(\mathrm{LSE}_{\mathrm{cache}}\bigr)
+ \exp\bigl(\mathrm{LSE}_{\mathrm{specs}}\bigr), \nonumber\\[0.5em]
o_{\mathrm{merge}} &= \frac{N_{\mathrm{num}}}{D_{\mathrm{den}}}.
\label{equ:all_attn}
\end{align}

By aggregating Equation \ref{equ:split_attn} into Equation \ref{equ:all_attn}, we can get the following equation:

\begin{align}
    o_{\mathrm{merge}} =& o_{\mathrm{cache}} \cdot\exp\bigl(\mathrm{LSE}_{\mathrm{cache}} - \mathrm{LSE}_{\mathrm{merge}}\bigr)
     \nonumber\\+& o_{\mathrm{specs}} \cdot \exp\bigl(\mathrm{LSE}_{\mathrm{specs}} - \mathrm{LSE}_{\mathrm{merge}}\bigr).
\end{align}

\end{proof}

\section{Experiments Details}
\label{appendix:training_details}

All models are trained using eight A100 80GB GPUs. For the 7B, 8B, and 13B target models trained on short-context data, we employ \textsc{LongSpec} with ZeRO-1~\cite{rasley2020deepspeed}. For the 7B, 8B, and 13B models trained on long-context data, as well as for all settings of the 33B target models, we utilize ZeRO-3. 

Standard cross-entropy is used to optimize the draft model while the parameters of the target model are kept frozen. To mitigate the VRAM peak caused by the computation of the logits, we use a fused-linear-and-cross-entropy loss implemented by the Liger Kernel~\cite{hsu2024liger}, which computes the LM head and the softmax function together and can greatly alleviate this problem. 

For the SlimPajama-6B dataset, we configure the batch size (including accumulation) to 2048, set the maximum learning rate to 5e-4 with a cosine learning rate schedule~\cite{loshchilov2017sgdr}, and optimize the draft model using AdamW~\cite{kingma2015adam}. When training on long-context datasets, we adopt a batch size of 256 and a maximum learning rate of 5e-6. The draft model is trained for only one epoch on all datasets.

It is important to note that the primary computational cost arises from forwarding the target model to obtain the KV cache. Recently, some companies have introduced a service known as context caching~\cite{deepseek2024contextcaching, gemini2024contextcaching}, which involves storing large volumes of KV cache. Consequently, in real-world deployment, these pre-stored KV caches can be directly utilized as training data, significantly accelerating the training process.

\begin{table*}[t]
\centering
\caption{Average acceptance length $\tau$ and decoding speed (tokens/s) across different models and settings. Specifically, ``Vanilla HF'' refers to HuggingFace’s PyTorch-based attention implementation, while ``Vanilla FA'' employs \texttt{Flash Attention}. All results are computed at $T=0$.}
\resizebox{\linewidth}{!}{
\begin{tabular}{cc|cccccccccc}
\toprule
& \multirow{2}{*}{Setting} 
& \multicolumn{2}{c}{GovReport} 
& \multicolumn{2}{c}{QMSum}
& \multicolumn{2}{c}{MultiNews}
& \multicolumn{2}{c}{LCC}
& \multicolumn{2}{c}{RB-P} \\
\cmidrule(lr){3-4} \cmidrule(lr){5-6} \cmidrule(lr){7-8} \cmidrule(lr){9-10} \cmidrule(lr){11-12}
& 
& $\tau$ & Tokens/s
& $\tau$ & Tokens/s
& $\tau$ & Tokens/s
& $\tau$ & Tokens/s
& $\tau$ & Tokens/s\\
\midrule

\multirow{4}{*}{\rotatebox{90}{V-7B}}
& Vanilla HF  
& 1.00 & 25.25
& 1.00 & 18.12
& 1.00 & 27.29
& 1.00 & 25.25
& 1.00 & 19.18\\

& Vanilla FA
& 1.00 & 45.76
& 1.00 & 43.68
& 1.00 & 55.99
& 1.00 & 54.07
& 1.00 & 46.61\\

& TR
& 2.83 & 94.06
& 2.13 & 68.23
& 2.81 & 94.51
& 2.72 & 87.77
& 2.83 & 94.10\\

& EAGLE
& 2.02 & 33.43
& 1.91 & 26.78
& 1.97 & 36.62
& 1.92 & 40.64
& 1.92 & 33.84\\

& LongSpec
& 3.57 & 102.23 
& 3.14 &  88.87
& 3.51 & 100.55
& 3.73 & 107.30
& 3.86 & 110.76\\
\midrule

\multirow{4}{*}{\rotatebox{90}{LC-7B}}
& Vanilla HF  
& 1.00 & 25.27
& 1.00 & 14.11
& 1.00 & 27.66
& 1.00 & 25.27
& 1.00 & 17.02 \\

& Vanilla FA  
& 1.00 & 42.14
& 1.00 & 36.87
& 1.00 & 50.19
& 1.00 & 54.17
& 1.00 & 42.69 \\

& TR
& 2.90 & 94.82 
& 2.20 & 64.96 
& 2.75 & 94.94 
& 2.80 & 96.67 
& 3.05 & 100.41\\

& EAGLE    
& 2.10 & 32.06
& 1.94 & 26.02
& 2.02 & 34.38
& 2.09 & 38.81
& 2.10 & 29.75\\

& LongSpec  
& 3.59 & 101.43
& 3.06 &  85.23
& 3.41 &  97.93
& 4.21 & 122.30
& 4.03 & 115.27\\





\bottomrule
\end{tabular}
}
\label{tab:eagle}
\end{table*}

For the tree decoding of \textsc{LongSpec}, we employ dynamic beam search to construct the tree. Previous studies have shown that beam search, while achieving high acceptance rates, suffers from slow processing speed in speculative decoding~\cite{du2024glide}. Our research identifies that this slowdown is primarily caused by KV cache movement. In traditional beam search, nodes that do not fall within the top-$k$ likelihood are discarded, a step that necessitates KV cache movement. However, in speculative decoding, discarding these nodes is unnecessary, as draft sequences are not required to maintain uniform lengths. Instead, we can simply halt the computation of descendant nodes for low-likelihood branches without removing them entirely. By adopting this approach, beam search attains strong performance without excessive computational overhead. In our experiments, the beam width is set to $[4, 16, 16, 16, 16]$ for each speculation step. All inference experiments in this study are conducted using float16 precision on a single A100 80GB GPU.

\section{Experimental Results of EAGLE and  Token Recycling on Long-Context Speculative Decoding}

In Table~\ref{tab:eagle}, we compare the average acceptance length $\tau$ and decoding speed (tokens/s) for two models under four settings: the baseline PyTorch implementation from HuggingFace (``Vanilla HF''), the same model with \texttt{Flash Attention} (``Vanilla FA''), Token Recycling~\citep{luo2025turning} (``TR'', a SoTA retrieval-based method), EAGLE~\citep{li2024eagle} (trained with anchor offset indices and inference with HuggingFace), and our LongSpec with hybrid tree attention. Across five datasets (GovReport, QMSum, MultiNews, LCC, and RB-P), Vanilla HF's decoding speeds are limited between 14 and 30 tokens/s, while switching to \texttt{Flash Attention} boosts speeds to about 50 tokens/s, a more than 2.5$\times$ speedup. 

EAGLE extends the acceptance length to around 2 and achieves 26–40 tokens/s, yielding a 30–50\% speedup over Vanilla HF. However, because EAGLE cannot leverage \texttt{Flash Attention}, its decoding speed remains substantially below that of Vanilla FA in every setting. As for TR, while it extends the acceptance length to around 3 (far larger than EAGLE) and achieves moderate acceleration on many tasks, it consistently underperforms LongSpec across the board.

In contrast, our LongSpec with hybrid tree attention achieves much higher decoding speeds of about 100 tokens/s across all models and datasets. This demonstrates that EAGLE’s incompatibility with \texttt{Flash Attention} fundamentally limits its decoding performance. Our hybrid tree attention preserves compatibility with \texttt{Flash Attention}, thus unlocking substantially higher decoding speed, underscoring the importance of combining tree-structured attention with SoTA long-context inference techniques such as \texttt{Flash Attention}.

\section{Performance Analysis with Varying Prefill Lengths}
\label{sec:prefill_length}

\begin{table*}[ht]
  \centering
  \caption{A detailed breakdown of performance as the prefill length increases, with LongChat-7B on GovReport.}
  \begin{tabular}{c|cccccc}
    \toprule
    Prefill Length      & 0–5k   & 5k–10k  & 10k–15k & 15k–20k & 20k–25k & 25k–32k \\
    \midrule
    Tokens/s  & 116.65 & 115.52  & 114.54  & 113.47  & 115.13  & 103.68  \\
    $\tau$  & 4.01   & 3.97    & 3.97    & 4.12    & 4.45    & 3.97    \\
    \midrule
    Draft time (ms)     & 8.91   & 8.92    & 8.93    & 8.98    & 9.13    & 9.25    \\
    Target time (ms)    & 25.63  & 25.66   & 25.61   & 27.30   & 29.08   & 30.89   \\
    Verify time (ms)    & 6.18   & 6.22    & 6.23    & 6.24    & 6.27    & 6.28    \\
    \bottomrule
  \end{tabular}
  \label{tab:prefill_length}
\end{table*}

In Table~\ref{tab:prefill_length}, we show a detailed breakdown of performance as the prefill length increases, with LongChat-7B on GovReport. Across the all token ranges, the generation speed remains remarkably stable, only dropping a little in the 25k-32k range. The average acceptance length remains consistent across all ranges, which indicates stable behavior in the number of tokens the system chooses to retain during generation. This stability suggests that the draft quality is unaffected by the length of the prefill, maintaining consistent output dynamics.

In terms of latency, the draft time increases only marginally, from 8.91 ms in the shortest context range to 9.25 ms in the longest, while target time shows a gradual increase from 25.63 ms to 30.89 ms, reflecting the added computational load of managing larger contexts. Verify time remains almost constant across all ranges, increasing only slightly from 6.18 ms to 6.28 ms.

Together, these results demonstrate that the system scales effectively with longer input contexts, maintaining high throughput and consistent drafting quality with only modest increases in latency. This highlights the practicality and robustness of our approach for real-world applications involving extended input sequences.

\section{Pseudo Code}
\label{appendix:pseudo-code}

Here we provide pseudo code for Anchor-Offset Indexing and Flash Noisy Training.

\begin{algorithm}[H]
\begin{algorithmic}[1]
\State \textbf{Input:} Sequence length $N$; Max length $\text{MAX\_LEN}$; Query states $q_s$.
\State \textbf{Output:} Query states with RoPE applied using modified indices.
\State $P \gets \{0, 1, \dots, N-1\}$ \Comment{Initial position indices}
\State $o \gets \text{RandomInt}(0, \text{MAX\_LEN} - N)$ \Comment{Generate random offset}
\State $P[4:] \mathrel{+}= o$ \Comment{Apply offset to indices after the first 4 anchors}
\State \quad \textit{// e.g., for $N=128, o=16257$, $P$ becomes $[0, 1, 2, 3, 16261, \dots, 16385]$}
\State \Return $\text{RoPE}(q_s, P)$
\captionof{algorithm}{Anchor-Offset Indexing}
\label{alg:anchor-offset-optimized}
\end{algorithmic}
\end{algorithm}

\newpage

\begin{algorithm}[H]
\begin{algorithmic}[1]
\State \textbf{Input:} Queries $Q$, Key cache $K$, Value cache $V$.
\State \textbf{Output:} Final attention output.
\State $j \gets \text{RandomInt}(1, 4)$ \Comment{Randomly select number of tokens to drop}
\Statex
\State \textit{//Perform attention on sliced inputs}
\State $Q' \gets Q[j:]$ \Comment{Drop first $j$ queries}
\State $K' \gets K[:-j]$ \Comment{Drop last $j$ keys from cache}
\State $V' \gets V[:-j]$ \Comment{Drop last $j$ values from cache}
\State $attn\_out \gets \text{FlashAttention}(Q', K', V')$
\Statex
\State \textit{//Pad the output to match original query length}
\State $padded\_out \gets \text{Concat}(\text{Zeros}(j), attn\_out)$
\State \Return $\text{OutputProjection}(padded\_out)$
\captionof{algorithm}{Flash Noisy Training}
\label{alg:flash-noisy-optimized}
\end{algorithmic}
\end{algorithm}

\section{Case Study}
\label{appendix:visualization}
Here we display some illustrative cases from GovReport on the Longchat-7B model, where tokens marked in blue indicate draft tokens accepted by the target model. Due to space limitations, the complete answer is not presented here.

\newpage

\begin{tcolorbox}
  \textcolor{blue}{The} \textcolor{black}{report} \textcolor{black}{discuss}\textcolor{blue}{es} \textcolor{blue}{the} \textcolor{blue}{use} \textcolor{blue}{of} \textcolor{blue}{op}\textcolor{black}{io}\textcolor{blue}{id} \textcolor{blue}{medic}\textcolor{blue}{ations} \textcolor{blue}{in} \textcolor{black}{health}\textcolor{blue}{care} \textcolor{blue}{and} \textcolor{black}{the} \textcolor{blue}{potential} \textcolor{blue}{ris}\textcolor{blue}{ks} \textcolor{blue}{associated} \textcolor{blue}{with} \textcolor{black}{their} \textcolor{black}{mis}\textcolor{blue}{use}\textcolor{blue}{.} \textcolor{blue}{Op}\textcolor{blue}{io}\textcolor{blue}{id} \textcolor{black}{medic}\textcolor{blue}{ations} \textcolor{blue}{are} \textcolor{blue}{used} \textcolor{blue}{to} \textcolor{blue}{treat} \textcolor{black}{pain} \textcolor{blue}{and} \textcolor{black}{can} \textcolor{blue}{also} \textcolor{blue}{be} \textcolor{blue}{used} \textcolor{blue}{to} \textcolor{blue}{treat} \textcolor{black}{other} \textcolor{black}{health} \textcolor{blue}{problems}\textcolor{blue}{,} \textcolor{blue}{such} \textcolor{blue}{as} \textcolor{black}{severe} \textcolor{blue}{c}\textcolor{blue}{ough}\textcolor{blue}{ing}\textcolor{blue}{.} \textcolor{black}{There} \textcolor{blue}{are} \textcolor{blue}{three} \textcolor{blue}{types} \textcolor{blue}{of} \textcolor{blue}{op}\textcolor{black}{io}\textcolor{blue}{id} \textcolor{blue}{medic}\textcolor{blue}{ations} \textcolor{blue}{that} \textcolor{blue}{are} \textcolor{black}{approved} \textcolor{blue}{for} \textcolor{blue}{use} \textcolor{blue}{in} \textcolor{blue}{the} \textcolor{blue}{treatment} \textcolor{black}{of} \textcolor{blue}{op}\textcolor{blue}{io}\textcolor{blue}{id} \textcolor{blue}{use} \textcolor{blue}{dis}\textcolor{black}{orders}\textcolor{blue}{:} \textcolor{black}{m}\textcolor{blue}{eth}\textcolor{blue}{ad}\textcolor{blue}{one}\textcolor{blue}{,} \textcolor{black}{bu}\textcolor{blue}{pr}\textcolor{blue}{en}\textcolor{blue}{orph}\textcolor{blue}{ine}\textcolor{blue}{,} \textcolor{black}{and} \textcolor{blue}{n}\textcolor{blue}{alt}\textcolor{blue}{re}\textcolor{blue}{x}\textcolor{blue}{one}\textcolor{black}{.} \textcolor{blue}{M}\textcolor{blue}{eth}\textcolor{blue}{ad}\textcolor{blue}{one} \textcolor{blue}{is} \textcolor{black}{a} \textcolor{blue}{full} \textcolor{black}{op}\textcolor{blue}{io}\textcolor{blue}{id} \textcolor{blue}{ag}\textcolor{blue}{on}\textcolor{blue}{ist}\textcolor{black}{,} \textcolor{blue}{meaning} \textcolor{blue}{it} \textcolor{black}{bind}\textcolor{blue}{s} \textcolor{blue}{to} \textcolor{black}{and} \textcolor{blue}{activ}\textcolor{blue}{ates} \textcolor{blue}{op}\textcolor{blue}{io}\textcolor{blue}{id} \textcolor{black}{re}\textcolor{blue}{cept}\textcolor{blue}{ors} \textcolor{blue}{in} \textcolor{blue}{the} \textcolor{blue}{body}\textcolor{black}{.} \textcolor{black}{Bu}\textcolor{blue}{pr}\textcolor{blue}{en}\textcolor{blue}{orph}\textcolor{blue}{ine} \textcolor{blue}{is} \textcolor{black}{a} \textcolor{blue}{partial} \textcolor{blue}{op}\textcolor{blue}{io}\textcolor{blue}{id} \textcolor{black}{ag}\textcolor{blue}{on}\textcolor{blue}{ist}\textcolor{blue}{,} \textcolor{blue}{meaning} \textcolor{blue}{it} \textcolor{black}{also} \textcolor{blue}{bind}\textcolor{blue}{s} \textcolor{blue}{to} \textcolor{black}{and} \textcolor{black}{activ}\textcolor{blue}{ates} \textcolor{blue}{op}\textcolor{blue}{io}\textcolor{blue}{id} \textcolor{blue}{re}\textcolor{black}{cept}\textcolor{blue}{ors}\textcolor{blue}{,} \textcolor{blue}{but} \textcolor{black}{to} \textcolor{blue}{a} \textcolor{blue}{less}\textcolor{blue}{er} \textcolor{blue}{extent} \textcolor{blue}{than} \textcolor{black}{m}\textcolor{blue}{eth}\textcolor{blue}{ad}\textcolor{blue}{one}\textcolor{blue}{.} \textcolor{black}{N}\textcolor{blue}{alt}\textcolor{blue}{re}\textcolor{blue}{x}\textcolor{blue}{one} \textcolor{blue}{is} \textcolor{black}{an} \textcolor{blue}{op}\textcolor{blue}{io}\textcolor{blue}{id} \textcolor{blue}{ant}\textcolor{blue}{agon}\textcolor{black}{ist}\textcolor{blue}{,} \textcolor{blue}{meaning} \textcolor{blue}{it} \textcolor{black}{bind}\textcolor{blue}{s} \textcolor{blue}{to} \textcolor{black}{and} \textcolor{black}{blocks} \textcolor{blue}{the} \textcolor{black}{effects} \textcolor{blue}{of} \textcolor{blue}{op}\textcolor{blue}{io}\textcolor{blue}{id} \textcolor{blue}{re}\textcolor{black}{cept}\textcolor{blue}{ors}\textcolor{blue}{.}
  
  \textcolor{black}{The} \textcolor{black}{report} \textcolor{black}{also} \textcolor{blue}{discuss}\textcolor{blue}{es} \textcolor{blue}{the} \textcolor{blue}{potential} \textcolor{blue}{ris}\textcolor{black}{ks} \textcolor{blue}{associated} \textcolor{blue}{with} \textcolor{blue}{the} \textcolor{blue}{use} \textcolor{black}{of} \textcolor{blue}{op}\textcolor{blue}{io}\textcolor{blue}{id} \textcolor{blue}{medic}\textcolor{blue}{ations}\textcolor{black}{,} \textcolor{blue}{including} \textcolor{blue}{the} \textcolor{black}{risk} \textcolor{blue}{of} \textcolor{blue}{add}\textcolor{blue}{iction} \textcolor{blue}{and} \textcolor{black}{the} \textcolor{blue}{risk} \textcolor{blue}{of} \textcolor{blue}{over}\textcolor{blue}{d}\textcolor{blue}{ose}\textcolor{black}{.} \textcolor{blue}{The} \textcolor{blue}{use} \textcolor{blue}{of} \textcolor{blue}{op}\textcolor{blue}{io}\textcolor{black}{id} \textcolor{blue}{medic}\textcolor{blue}{ations} \textcolor{blue}{can} \textcolor{blue}{lead} \textcolor{blue}{to} \textcolor{black}{physical} \textcolor{blue}{dependence} \textcolor{blue}{and} \textcolor{black}{toler}\textcolor{blue}{ance}\textcolor{blue}{,} \textcolor{blue}{which} \textcolor{blue}{can} \textcolor{black}{make} \textcolor{blue}{it} \textcolor{blue}{difficult} \textcolor{blue}{to} \textcolor{black}{stop} \textcolor{blue}{using} \textcolor{blue}{the} \textcolor{blue}{medic}\textcolor{blue}{ation}\textcolor{blue}{.} \textcolor{black}{Additionally}\textcolor{blue}{,} \textcolor{blue}{the} \textcolor{black}{mis}\textcolor{blue}{use} \textcolor{blue}{of} \textcolor{blue}{op}\textcolor{blue}{io}\textcolor{blue}{id} \textcolor{black}{medic}\textcolor{blue}{ations} \textcolor{blue}{can} \textcolor{blue}{lead} \textcolor{blue}{to} \textcolor{black}{add}\textcolor{blue}{iction}\textcolor{blue}{,} \textcolor{blue}{which} \textcolor{blue}{can} \textcolor{black}{have} \textcolor{blue}{serious} \textcolor{blue}{consequences} \textcolor{blue}{for} \textcolor{blue}{the} \textcolor{black}{individual} \textcolor{blue}{and} \textcolor{blue}{their} \textcolor{black}{loved} \textcolor{blue}{ones}\textcolor{blue}{.}
  
  \textcolor{blue}{The} \textcolor{black}{report} \textcolor{blue}{also} \textcolor{blue}{discuss}\textcolor{blue}{es} \textcolor{blue}{the} \textcolor{blue}{potential} \textcolor{black}{ris}\textcolor{blue}{ks} \textcolor{blue}{associated} \textcolor{blue}{with} \textcolor{blue}{the} \textcolor{black}{di}\textcolor{blue}{version} \textcolor{blue}{of} \textcolor{blue}{op}\textcolor{blue}{io}\textcolor{blue}{id} \textcolor{black}{medic}\textcolor{blue}{ations}\textcolor{blue}{,} \textcolor{blue}{which} \textcolor{black}{is} \textcolor{blue}{the} \textcolor{black}{illegal} \textcolor{blue}{use} \textcolor{blue}{of} \textcolor{blue}{pres}\textcolor{blue}{cription} \textcolor{blue}{op}\textcolor{black}{io}\textcolor{blue}{ids} \textcolor{blue}{for} \textcolor{black}{non}\textcolor{blue}{-}\textcolor{blue}{med}\textcolor{blue}{ical} \textcolor{blue}{purposes}\textcolor{blue}{.} \textcolor{black}{D}\textcolor{blue}{ivers}\textcolor{blue}{ion} \textcolor{black}{can} \textcolor{blue}{lead} \textcolor{blue}{to} \textcolor{blue}{increased} \textcolor{black}{rates} \textcolor{blue}{of} \textcolor{blue}{add}\textcolor{blue}{iction}\textcolor{blue}{,} \textcolor{black}{over}\textcolor{blue}{d}\textcolor{blue}{ose}\textcolor{blue}{,} \textcolor{blue}{and} \textcolor{blue}{death}\textcolor{black}{.}
  
  \textcolor{blue}{The} \textcolor{blue}{report} \textcolor{black}{con}\textcolor{blue}{cludes} \textcolor{blue}{by} \textcolor{black}{discuss}\textcolor{blue}{ing} \textcolor{blue}{the} \textcolor{blue}{importance} \textcolor{blue}{of} \textcolor{black}{proper} \textcolor{blue}{use} \textcolor{blue}{and} \textcolor{black}{monitoring} \textcolor{blue}{of} \textcolor{blue}{op}\textcolor{blue}{io}\textcolor{blue}{id} \textcolor{blue}{medic}\textcolor{black}{ations}\textcolor{blue}{,} \textcolor{blue}{as} \textcolor{blue}{well} \textcolor{blue}{as} \textcolor{blue}{the} \textcolor{black}{need} \textcolor{blue}{for} \textcolor{black}{continued} \textcolor{blue}{research} \textcolor{blue}{and} \textcolor{black}{development} \textcolor{blue}{of} \textcolor{blue}{new} \textcolor{blue}{treat}\textcolor{blue}{ments} \textcolor{black}{for} \textcolor{blue}{op}\textcolor{blue}{io}\textcolor{blue}{id} \textcolor{blue}{use} \textcolor{black}{dis}\textcolor{blue}{orders}\textcolor{blue}{.}
\end{tcolorbox}

\begin{tcolorbox}
    \textcolor{blue}{The} \textcolor{blue}{Rail}\textcolor{black}{road} \textcolor{blue}{Ret}\textcolor{blue}{irement} \textcolor{blue}{Board} \textcolor{black}{(}\textcolor{blue}{R}\textcolor{blue}{R}\textcolor{blue}{B}\textcolor{blue}{)} \textcolor{blue}{is} \textcolor{black}{an} \textcolor{blue}{independent} \textcolor{blue}{federal} \textcolor{blue}{ag}\textcolor{blue}{ency} \textcolor{blue}{that} \textcolor{black}{admin}\textcolor{blue}{ister}\textcolor{blue}{s} \textcolor{black}{ret}\textcolor{blue}{irement}\textcolor{blue}{,} \textcolor{black}{surv}\textcolor{blue}{iv}\textcolor{blue}{or}\textcolor{blue}{,} \textcolor{black}{dis}\textcolor{blue}{ability}\textcolor{blue}{,} \textcolor{blue}{un}\textcolor{black}{emp}\textcolor{blue}{loyment}\textcolor{blue}{,} \textcolor{blue}{and} \textcolor{blue}{sick}\textcolor{blue}{ness} \textcolor{black}{ins}\textcolor{blue}{urance} \textcolor{blue}{for} \textcolor{black}{rail}\textcolor{blue}{road} \textcolor{blue}{workers} \textcolor{blue}{and} \textcolor{blue}{their} \textcolor{blue}{families}\textcolor{black}{.} \textcolor{blue}{The} \textcolor{blue}{R}\textcolor{blue}{R}\textcolor{blue}{B} \textcolor{black}{covers} \textcolor{black}{workers} \textcolor{blue}{who} \textcolor{blue}{are} \textcolor{blue}{employed} \textcolor{blue}{by} \textcolor{blue}{rail}\textcolor{black}{ro}\textcolor{blue}{ads} \textcolor{black}{engaged} \textcolor{blue}{in} \textcolor{black}{inter}\textcolor{blue}{state} \textcolor{blue}{commerce} \textcolor{blue}{and} \textcolor{black}{related} \textcolor{black}{subs}\textcolor{blue}{idi}\textcolor{blue}{aries}\textcolor{blue}{,} \textcolor{black}{rail}\textcolor{blue}{road} \textcolor{blue}{associations}\textcolor{blue}{,} \textcolor{blue}{and} \textcolor{blue}{rail}\textcolor{black}{road} \textcolor{blue}{labor} \textcolor{blue}{organizations}\textcolor{blue}{.}
    
    \textcolor{blue}{The} \textcolor{black}{R}\textcolor{blue}{R}\textcolor{blue}{B} \textcolor{blue}{has} \textcolor{blue}{two} \textcolor{blue}{main} \textcolor{black}{programs}\textcolor{blue}{:} \textcolor{blue}{the} \textcolor{blue}{Rail}\textcolor{blue}{road} \textcolor{blue}{Ret}\textcolor{black}{irement} \textcolor{black}{Act} \textcolor{blue}{(}\textcolor{blue}{R}\textcolor{blue}{RA}\textcolor{blue}{)} \textcolor{black}{and} \textcolor{blue}{the} \textcolor{blue}{Rail}\textcolor{blue}{road} \textcolor{black}{Un}\textcolor{blue}{emp}\textcolor{blue}{loyment} \textcolor{blue}{In}\textcolor{blue}{sur}\textcolor{blue}{ance} \textcolor{black}{Act} \textcolor{blue}{(}\textcolor{blue}{R}\textcolor{black}{UI}\textcolor{blue}{A}\textcolor{blue}{).} \textcolor{blue}{The} \textcolor{blue}{R}\textcolor{blue}{RA} \textcolor{black}{author}\textcolor{blue}{izes} \textcolor{blue}{ret}\textcolor{blue}{irement}\textcolor{black}{,} \textcolor{black}{surv}\textcolor{blue}{iv}\textcolor{blue}{or}\textcolor{blue}{,} \textcolor{blue}{and} \textcolor{blue}{dis}\textcolor{black}{ability} \textcolor{blue}{benefits} \textcolor{blue}{for} \textcolor{blue}{rail}\textcolor{blue}{road} \textcolor{blue}{workers} \textcolor{black}{and} \textcolor{blue}{their} \textcolor{blue}{families}\textcolor{blue}{.} \textcolor{blue}{The} \textcolor{blue}{R}\textcolor{black}{UI}\textcolor{blue}{A} \textcolor{black}{provides} \textcolor{black}{un}\textcolor{blue}{emp}\textcolor{blue}{loyment} \textcolor{blue}{and} \textcolor{black}{sick}\textcolor{blue}{ness} \textcolor{blue}{benefits} \textcolor{blue}{for} \textcolor{blue}{rail}\textcolor{blue}{road} \textcolor{black}{workers}\textcolor{blue}{.}
    
    \textcolor{blue}{The} \textcolor{black}{number} \textcolor{blue}{of} \textcolor{blue}{rail}\textcolor{blue}{road} \textcolor{blue}{workers} \textcolor{black}{has} \textcolor{black}{been} \textcolor{blue}{decl}\textcolor{blue}{ining} \textcolor{blue}{since} \textcolor{blue}{the} \textcolor{blue}{}\textcolor{black}{1}\textcolor{blue}{9}\textcolor{blue}{5}\textcolor{blue}{0}\textcolor{blue}{s}\textcolor{blue}{,} \textcolor{black}{although} \textcolor{blue}{the} \textcolor{blue}{rate} \textcolor{blue}{of} \textcolor{black}{decl}\textcolor{blue}{ine} \textcolor{blue}{has} \textcolor{blue}{been} \textcolor{black}{irregular}\textcolor{blue}{.} \textcolor{blue}{In} \textcolor{blue}{recent} \textcolor{blue}{years}\textcolor{blue}{,} \textcolor{black}{rail}\textcolor{blue}{road} \textcolor{blue}{employ}\textcolor{blue}{ment} \textcolor{blue}{has} \textcolor{blue}{increased} \textcolor{black}{after} \textcolor{black}{reaching} \textcolor{blue}{an} \textcolor{blue}{all}\textcolor{blue}{-}\textcolor{blue}{time} \textcolor{blue}{low} \textcolor{black}{of} \textcolor{blue}{}\textcolor{blue}{2}\textcolor{black}{1}\textcolor{blue}{5}\textcolor{blue}{,}\textcolor{blue}{0}\textcolor{blue}{0}\textcolor{blue}{0} \textcolor{black}{workers} \textcolor{blue}{in} \textcolor{blue}{January} \textcolor{blue}{}\textcolor{blue}{2}\textcolor{blue}{0}\textcolor{black}{1}\textcolor{blue}{0}\textcolor{blue}{.} \textcolor{blue}{In} \textcolor{black}{April} \textcolor{blue}{}\textcolor{blue}{2}\textcolor{blue}{0}\textcolor{blue}{1}\textcolor{blue}{5}\textcolor{black}{,} \textcolor{blue}{rail}\textcolor{blue}{road} \textcolor{blue}{employ}\textcolor{blue}{ment} \textcolor{black}{pe}\textcolor{blue}{aked} \textcolor{blue}{at} \textcolor{blue}{}\textcolor{blue}{2}\textcolor{black}{5}\textcolor{black}{3}\textcolor{blue}{,}\textcolor{blue}{0}\textcolor{blue}{0}\textcolor{blue}{0} \textcolor{blue}{workers}\textcolor{black}{,} \textcolor{blue}{the} \textcolor{blue}{highest} \textcolor{blue}{level} \textcolor{blue}{since} \textcolor{black}{November} \textcolor{blue}{}\textcolor{blue}{1}\textcolor{blue}{9}\textcolor{blue}{9}\textcolor{blue}{9}\textcolor{black}{,} \textcolor{blue}{and} \textcolor{blue}{then} \textcolor{black}{decl}\textcolor{blue}{ined} \textcolor{black}{through} \textcolor{black}{F}\textcolor{blue}{Y}\textcolor{blue}{2}\textcolor{blue}{0}\textcolor{blue}{1}\textcolor{blue}{7}\textcolor{black}{,} \textcolor{blue}{falling} \textcolor{blue}{to} \textcolor{blue}{}\textcolor{blue}{2}\textcolor{blue}{2}\textcolor{black}{1}\textcolor{blue}{,}\textcolor{blue}{0}\textcolor{blue}{0}\textcolor{blue}{0} \textcolor{black}{workers}\textcolor{blue}{.}

    \textcolor{blue}{The} \textcolor{black}{R}\textcolor{blue}{R}\textcolor{blue}{B}\textcolor{blue}{'}\textcolor{blue}{s} \textcolor{black}{programs} \textcolor{blue}{are} \textcolor{blue}{designed} \textcolor{blue}{to} \textcolor{blue}{provide} \textcolor{black}{compreh}\textcolor{blue}{ensive} \textcolor{blue}{benefits} \textcolor{blue}{to} \textcolor{blue}{rail}\textcolor{blue}{road} \textcolor{black}{workers} \textcolor{blue}{and} \textcolor{blue}{their} \textcolor{blue}{families}\textcolor{blue}{.} \textcolor{blue}{The} \textcolor{black}{R}\textcolor{blue}{RA} \textcolor{blue}{and} \textcolor{black}{R}\textcolor{blue}{UI}\textcolor{blue}{A} \textcolor{black}{are} \textcolor{black}{important} \textcolor{blue}{components} \textcolor{blue}{of} \textcolor{blue}{the} \textcolor{blue}{rail}\textcolor{blue}{road} \textcolor{black}{industry}\textcolor{blue}{'}\textcolor{blue}{s} \textcolor{blue}{ret}\textcolor{blue}{irement} \textcolor{blue}{and} \textcolor{black}{benefits} \textcolor{blue}{system}\textcolor{blue}{.} \textcolor{blue}{The} \textcolor{blue}{R}\textcolor{blue}{R}\textcolor{black}{B}\textcolor{blue}{'}\textcolor{blue}{s} \textcolor{black}{efforts} \textcolor{blue}{to} \textcolor{blue}{maintain} \textcolor{black}{and} \textcolor{blue}{improve} \textcolor{blue}{these} \textcolor{black}{programs} \textcolor{blue}{are} \textcolor{black}{cru}\textcolor{blue}{cial} \textcolor{blue}{for} \textcolor{blue}{the} \textcolor{black}{well}\textcolor{blue}{-}\textcolor{blue}{be}\textcolor{blue}{ing} \textcolor{blue}{of} \textcolor{blue}{rail}\textcolor{black}{road} \textcolor{blue}{workers} \textcolor{blue}{and} \textcolor{blue}{their} \textcolor{blue}{families}\textcolor{blue}{.}
\end{tcolorbox}

\begin{tcolorbox}
    \textcolor{blue}{The} \textcolor{blue}{Cong}\textcolor{black}{r}\textcolor{blue}{essional} \textcolor{blue}{Gold} \textcolor{blue}{Medal} \textcolor{blue}{is} \textcolor{blue}{a} \textcolor{black}{pr}\textcolor{blue}{estig}\textcolor{blue}{ious} \textcolor{blue}{award} \textcolor{blue}{given} \textcolor{black}{by} \textcolor{blue}{the} \textcolor{blue}{United} \textcolor{blue}{States} \textcolor{blue}{Congress} \textcolor{black}{to} \textcolor{black}{individuals} \textcolor{blue}{and} \textcolor{blue}{groups} \textcolor{blue}{in} \textcolor{blue}{recognition} \textcolor{blue}{of} \textcolor{black}{their} \textcolor{blue}{distinguished} \textcolor{blue}{contributions}\textcolor{black}{,} \textcolor{blue}{achiev}\textcolor{blue}{ements}\textcolor{blue}{,} \textcolor{blue}{and} \textcolor{black}{services} \textcolor{blue}{to} \textcolor{blue}{the} \textcolor{black}{country}\textcolor{blue}{.} \textcolor{blue}{The} \textcolor{black}{tradition} \textcolor{blue}{of} \textcolor{blue}{award}\textcolor{blue}{ing} \textcolor{black}{gold} \textcolor{blue}{med}\textcolor{blue}{als} \textcolor{black}{dates} \textcolor{blue}{back} \textcolor{blue}{to} \textcolor{blue}{the} \textcolor{blue}{late} \textcolor{blue}{}\textcolor{black}{1}\textcolor{blue}{8}\textcolor{blue}{th} \textcolor{blue}{century}\textcolor{blue}{,} \textcolor{black}{and} \textcolor{black}{it} \textcolor{blue}{has} \textcolor{blue}{been} \textcolor{black}{used} \textcolor{blue}{to} \textcolor{blue}{honor} \textcolor{black}{a} \textcolor{blue}{wide} \textcolor{blue}{range} \textcolor{blue}{of} \textcolor{blue}{individuals}\textcolor{blue}{,} \textcolor{black}{including} \textcolor{blue}{military} \textcolor{blue}{leaders}\textcolor{blue}{,} \textcolor{black}{scient}\textcolor{blue}{ists}\textcolor{blue}{,} \textcolor{blue}{artists}\textcolor{blue}{,} \textcolor{blue}{and} \textcolor{black}{human}\textcolor{black}{it}\textcolor{blue}{ari}\textcolor{blue}{ans}\textcolor{blue}{.}
    
    \textcolor{blue}{The} \textcolor{blue}{first} \textcolor{black}{Cong}\textcolor{blue}{r}\textcolor{blue}{essional} \textcolor{blue}{Gold} \textcolor{blue}{Med}\textcolor{blue}{als} \textcolor{black}{were} \textcolor{blue}{issued} \textcolor{blue}{by} \textcolor{blue}{the} \textcolor{black}{Cont}\textcolor{blue}{inental} \textcolor{blue}{Congress} \textcolor{blue}{in} \textcolor{blue}{the} \textcolor{blue}{late} \textcolor{black}{}\textcolor{blue}{1}\textcolor{blue}{7}\textcolor{blue}{0}\textcolor{blue}{0}\textcolor{blue}{s}\textcolor{black}{,} \textcolor{blue}{and} \textcolor{black}{since} \textcolor{blue}{then}\textcolor{blue}{,} \textcolor{blue}{Congress} \textcolor{blue}{has} \textcolor{black}{awarded} \textcolor{black}{over} \textcolor{blue}{}\textcolor{blue}{2}\textcolor{blue}{,}\textcolor{blue}{0}\textcolor{blue}{0}\textcolor{black}{0} \textcolor{blue}{med}\textcolor{blue}{als} \textcolor{blue}{to} \textcolor{black}{various} \textcolor{blue}{individuals} \textcolor{blue}{and} \textcolor{blue}{groups}\textcolor{blue}{.} \textcolor{blue}{The} \textcolor{black}{award}\textcolor{black}{ing} \textcolor{blue}{of} \textcolor{blue}{the} \textcolor{black}{Cong}\textcolor{blue}{r}\textcolor{blue}{essional} \textcolor{blue}{Gold} \textcolor{blue}{Medal} \textcolor{blue}{is} \textcolor{black}{not} \textcolor{blue}{a} \textcolor{black}{permanent} \textcolor{blue}{stat}\textcolor{blue}{ut}\textcolor{blue}{ory} \textcolor{black}{provision}\textcolor{blue}{,} \textcolor{blue}{and} \textcolor{blue}{it} \textcolor{blue}{is} \textcolor{black}{typically} \textcolor{blue}{done} \textcolor{blue}{through} \textcolor{black}{special} \textcolor{blue}{legisl}\textcolor{blue}{ation}\textcolor{blue}{.}
    
    \textcolor{blue}{The} \textcolor{black}{process} \textcolor{blue}{of} \textcolor{blue}{award}\textcolor{blue}{ing} \textcolor{blue}{the} \textcolor{black}{Cong}\textcolor{blue}{r}\textcolor{blue}{essional} \textcolor{blue}{Gold} \textcolor{blue}{Medal} \textcolor{black}{involves} \textcolor{black}{several} \textcolor{blue}{steps}\textcolor{blue}{,} \textcolor{blue}{including} \textcolor{black}{the} \textcolor{blue}{introduction} \textcolor{blue}{of} \textcolor{blue}{legisl}\textcolor{blue}{ation}\textcolor{black}{,} \textcolor{blue}{the} \textcolor{black}{consideration} \textcolor{blue}{of} \textcolor{blue}{the} \textcolor{blue}{legisl}\textcolor{blue}{ation} \textcolor{black}{by} \textcolor{blue}{the} \textcolor{black}{relevant} \textcolor{blue}{comm}\textcolor{blue}{itte}\textcolor{blue}{es}\textcolor{blue}{,} \textcolor{blue}{and} \textcolor{black}{the} \textcolor{blue}{appro}\textcolor{blue}{val} \textcolor{blue}{of} \textcolor{blue}{the} \textcolor{blue}{legisl}\textcolor{black}{ation} \textcolor{blue}{by} \textcolor{blue}{both} \textcolor{blue}{the} \textcolor{black}{House} \textcolor{blue}{of} \textcolor{blue}{Representatives} \textcolor{blue}{and} \textcolor{black}{the} \textcolor{blue}{Senate}\textcolor{blue}{.} \textcolor{blue}{Once} \textcolor{black}{the} \textcolor{blue}{legisl}\textcolor{blue}{ation} \textcolor{blue}{is} \textcolor{blue}{approved}\textcolor{blue}{,} \textcolor{black}{the} \textcolor{blue}{Secretary} \textcolor{blue}{of} \textcolor{blue}{the} \textcolor{blue}{Tre}\textcolor{blue}{as}\textcolor{black}{ury} \textcolor{black}{is} \textcolor{blue}{responsible} \textcolor{blue}{for} \textcolor{black}{striking} \textcolor{blue}{the} \textcolor{blue}{medal}\textcolor{blue}{,} \textcolor{black}{which} \textcolor{blue}{is} \textcolor{blue}{then} \textcolor{black}{presented}
    
    \textcolor{blue}{The} \textcolor{black}{design} \textcolor{blue}{of} \textcolor{blue}{the} \textcolor{blue}{Cong}\textcolor{blue}{r}\textcolor{blue}{essional} \textcolor{black}{Gold} \textcolor{blue}{Medal} \textcolor{blue}{is} \textcolor{blue}{typically} \textcolor{black}{determined} \textcolor{blue}{by} \textcolor{blue}{the} \textcolor{blue}{Secretary} \textcolor{blue}{of} \textcolor{blue}{the} \textcolor{black}{Tre}\textcolor{blue}{as}\textcolor{blue}{ury}\textcolor{blue}{,} \textcolor{black}{in} \textcolor{blue}{consult}\textcolor{blue}{ation} \textcolor{blue}{with} \textcolor{blue}{the} \textcolor{black}{Cit}\textcolor{blue}{iz}\textcolor{blue}{ens} \textcolor{black}{Co}\textcolor{blue}{in}\textcolor{black}{age} \textcolor{blue}{Ad}\textcolor{blue}{vis}\textcolor{blue}{ory} \textcolor{blue}{Committee} \textcolor{blue}{and} \textcolor{black}{the} \textcolor{black}{Commission} \textcolor{black}{of} \textcolor{blue}{Fine} \textcolor{blue}{Arts}\textcolor{blue}{.} \textcolor{blue}{The} \textcolor{blue}{medal} \textcolor{black}{typically} \textcolor{blue}{features} \textcolor{blue}{a} \textcolor{black}{portrait} \textcolor{blue}{of} \textcolor{blue}{the} \textcolor{blue}{recip}\textcolor{blue}{ient}\textcolor{black}{,} \textcolor{black}{as} \textcolor{blue}{well} \textcolor{blue}{as} \textcolor{blue}{ins}\textcolor{black}{cri}\textcolor{blue}{ptions} \textcolor{blue}{and} \textcolor{black}{symbols} \textcolor{blue}{that} \textcolor{blue}{reflect} \textcolor{blue}{the} \textcolor{black}{recip}\textcolor{blue}{ient}\textcolor{blue}{'}\textcolor{blue}{s} \textcolor{blue}{achiev}\textcolor{blue}{ements} \textcolor{black}{and} \textcolor{blue}{contributions}\textcolor{blue}{.}
    
    \textcolor{black}{The} \textcolor{blue}{Cong}\textcolor{blue}{r}\textcolor{blue}{essional} \textcolor{blue}{Gold} \textcolor{blue}{Medal} \textcolor{black}{is} \textcolor{blue}{considered} \textcolor{blue}{one} \textcolor{blue}{of} \textcolor{blue}{the} \textcolor{black}{highest} \textcolor{blue}{civil}\textcolor{blue}{ian} \textcolor{blue}{hon}\textcolor{blue}{ors} \textcolor{blue}{in} \textcolor{black}{the} \textcolor{blue}{United} \textcolor{blue}{States}\textcolor{blue}{,} \textcolor{blue}{and} \textcolor{blue}{it} \textcolor{black}{is} \textcolor{blue}{often} \textcolor{black}{given} \textcolor{blue}{to} \textcolor{blue}{individuals} \textcolor{blue}{who} \textcolor{blue}{have} \textcolor{black}{made} \textcolor{blue}{significant} \textcolor{blue}{contributions} \textcolor{blue}{to} \textcolor{black}{their} \textcolor{black}{field} \textcolor{blue}{or} \textcolor{blue}{to} \textcolor{black}{the} \textcolor{blue}{country} \textcolor{black}{as} \textcolor{blue}{a} \textcolor{blue}{whole}\textcolor{blue}{.} \textcolor{blue}{The} \textcolor{blue}{award} \textcolor{black}{has} \textcolor{blue}{been} \textcolor{blue}{given} \textcolor{blue}{to} \textcolor{black}{a} \textcolor{blue}{wide} \textcolor{blue}{range} \textcolor{blue}{of} \textcolor{blue}{individuals}\textcolor{blue}{,} \textcolor{black}{including} \textcolor{black}{military} \textcolor{blue}{hero}\textcolor{blue}{es}\textcolor{blue}{,} \textcolor{black}{civil} \textcolor{blue}{rights} \textcolor{blue}{leaders}\textcolor{blue}{,} \textcolor{blue}{and} \textcolor{black}{artists}\textcolor{blue}{.}
    
    \textcolor{black}{In} \textcolor{blue}{recent} \textcolor{blue}{years}\textcolor{blue}{,} \textcolor{blue}{the} \textcolor{black}{number} \textcolor{blue}{of} \textcolor{blue}{Cong}\textcolor{blue}{r}\textcolor{blue}{essional} \textcolor{blue}{Gold} \textcolor{black}{Med}\textcolor{blue}{als} \textcolor{blue}{awarded} \textcolor{black}{has} \textcolor{blue}{increased}\textcolor{blue}{,} \textcolor{black}{with} \textcolor{black}{over} \textcolor{blue}{}\textcolor{blue}{5}\textcolor{blue}{0} \textcolor{black}{b}\textcolor{blue}{ills} \textcolor{blue}{introduced} \textcolor{black}{in} \textcolor{blue}{the} \textcolor{blue}{}\textcolor{blue}{1}\textcolor{blue}{1}\textcolor{black}{3}\textcolor{blue}{th} \textcolor{blue}{Congress} \textcolor{black}{alone}\textcolor{blue}{.} \textcolor{blue}{The} \textcolor{blue}{award} \textcolor{blue}{has} \textcolor{black}{also} \textcolor{blue}{been} \textcolor{blue}{given} \textcolor{blue}{to} \textcolor{black}{a} \textcolor{black}{growing} \textcolor{blue}{number} \textcolor{blue}{of} \textcolor{blue}{groups}\textcolor{blue}{,} \textcolor{blue}{including} \textcolor{black}{military} \textcolor{black}{units}\textcolor{blue}{,} \textcolor{blue}{organizations}\textcolor{blue}{,} \textcolor{blue}{and} \textcolor{black}{even} \textcolor{black}{entire} \textcolor{blue}{cities}\textcolor{blue}{.}

\end{tcolorbox}

\begin{tcolorbox}
    \textcolor{blue}{The} \textcolor{blue}{report} \textcolor{black}{provides} \textcolor{blue}{an} \textcolor{blue}{over}\textcolor{blue}{view} \textcolor{blue}{of} \textcolor{blue}{the} \textcolor{black}{annual} \textcolor{blue}{appropri}\textcolor{blue}{ations} \textcolor{blue}{for} \textcolor{blue}{the} \textcolor{blue}{Department} \textcolor{black}{of} \textcolor{blue}{Hom}\textcolor{blue}{eland} \textcolor{blue}{Security} \textcolor{blue}{(}\textcolor{blue}{D}\textcolor{black}{HS}\textcolor{blue}{)} \textcolor{blue}{for} \textcolor{blue}{F}\textcolor{blue}{Y}\textcolor{blue}{2}\textcolor{black}{0}\textcolor{blue}{1}\textcolor{blue}{9}\textcolor{blue}{.} \textcolor{blue}{It} \textcolor{black}{comp}\textcolor{blue}{ares} \textcolor{blue}{the} \textcolor{black}{en}\textcolor{blue}{act}\textcolor{blue}{ed} \textcolor{blue}{F}\textcolor{blue}{Y}\textcolor{blue}{2}\textcolor{black}{0}\textcolor{blue}{1}\textcolor{blue}{8} \textcolor{blue}{appropri}\textcolor{blue}{ations} \textcolor{blue}{for} \textcolor{black}{D}\textcolor{blue}{HS}\textcolor{blue}{,} \textcolor{blue}{the} \textcolor{black}{Trump} \textcolor{blue}{Administration}\textcolor{blue}{'}\textcolor{blue}{s} \textcolor{blue}{F}\textcolor{black}{Y}\textcolor{blue}{2}\textcolor{blue}{0}\textcolor{blue}{1}\textcolor{blue}{9} \textcolor{blue}{budget} \textcolor{black}{request}\textcolor{blue}{,} \textcolor{blue}{and} \textcolor{blue}{the} \textcolor{black}{appropri}\textcolor{blue}{ations} \textcolor{blue}{measures} \textcolor{black}{developed} \textcolor{blue}{and} \textcolor{black}{considered} \textcolor{blue}{by} \textcolor{blue}{Congress} \textcolor{blue}{in} \textcolor{black}{response} \textcolor{blue}{to} \textcolor{blue}{the} \textcolor{black}{request}\textcolor{blue}{.} \textcolor{blue}{The} \textcolor{blue}{report} \textcolor{black}{ident}\textcolor{blue}{ifies} \textcolor{blue}{additional} \textcolor{black}{inform}\textcolor{blue}{ational} \textcolor{blue}{resources}\textcolor{blue}{,} \textcolor{black}{reports}\textcolor{blue}{,} \textcolor{blue}{and} \textcolor{black}{policy} \textcolor{blue}{exper}\textcolor{blue}{ts} \textcolor{blue}{that} \textcolor{black}{can} \textcolor{blue}{provide} \textcolor{black}{further} \textcolor{blue}{information} \textcolor{blue}{on} \textcolor{blue}{D}\textcolor{blue}{HS} \textcolor{black}{appropri}\textcolor{blue}{ations}\textcolor{blue}{.}
    
    \textcolor{blue}{The} \textcolor{blue}{report} \textcolor{black}{explains} \textcolor{black}{several} \textcolor{black}{special}\textcolor{blue}{ized} \textcolor{blue}{budget}\textcolor{blue}{ary} \textcolor{black}{concepts}\textcolor{blue}{,} \textcolor{blue}{including} \textcolor{blue}{budget} \textcolor{black}{authority}\textcolor{blue}{,} \textcolor{black}{oblig}\textcolor{blue}{ations}\textcolor{blue}{,} \textcolor{blue}{out}\textcolor{black}{l}\textcolor{blue}{ays}\textcolor{blue}{,} \textcolor{blue}{dis}\textcolor{blue}{cret}\textcolor{blue}{ion}\textcolor{black}{ary} \textcolor{black}{and} \textcolor{black}{mand}\textcolor{blue}{atory} \textcolor{blue}{sp}\textcolor{blue}{ending}\textcolor{blue}{,} \textcolor{black}{offset}\textcolor{blue}{ting} \textcolor{black}{collections}\textcolor{blue}{,} \textcolor{blue}{alloc}\textcolor{black}{ations}\textcolor{blue}{,} \textcolor{blue}{and} \textcolor{black}{adjust}\textcolor{blue}{ments} \textcolor{blue}{to} \textcolor{blue}{the} \textcolor{black}{dis}\textcolor{blue}{cret}\textcolor{blue}{ion}\textcolor{blue}{ary} \textcolor{blue}{sp}\textcolor{blue}{ending} \textcolor{black}{caps} \textcolor{black}{under} \textcolor{blue}{the} \textcolor{black}{Bud}\textcolor{blue}{get} \textcolor{blue}{Control} \textcolor{blue}{Act} \textcolor{blue}{(}\textcolor{black}{BC}\textcolor{blue}{A}\textcolor{blue}{).} \textcolor{blue}{It} \textcolor{blue}{also} \textcolor{black}{provides} \textcolor{blue}{a} \textcolor{blue}{detailed} \textcolor{blue}{analysis} \textcolor{blue}{of} \textcolor{blue}{the} \textcolor{black}{appropri}\textcolor{blue}{ations} \textcolor{blue}{process} \textcolor{blue}{for} \textcolor{blue}{D}\textcolor{blue}{HS}\textcolor{black}{,} \textcolor{blue}{including} \textcolor{blue}{the} \textcolor{black}{various} \textcolor{blue}{comm}\textcolor{blue}{itte}\textcolor{blue}{es} \textcolor{blue}{and} \textcolor{black}{sub}\textcolor{blue}{comm}\textcolor{blue}{itte}\textcolor{blue}{es} \textcolor{blue}{involved}\textcolor{blue}{,} \textcolor{black}{and} \textcolor{blue}{the} \textcolor{blue}{role} \textcolor{blue}{of} \textcolor{blue}{the} \textcolor{black}{Cong}\textcolor{blue}{r}\textcolor{blue}{essional} \textcolor{blue}{Bud}\textcolor{blue}{get} \textcolor{blue}{Office} \textcolor{black}{(}\textcolor{blue}{C}\textcolor{blue}{BO}\textcolor{blue}{)} \textcolor{blue}{and} \textcolor{blue}{the} \textcolor{black}{Government} \textcolor{blue}{Account}\textcolor{blue}{ability} \textcolor{blue}{Office} \textcolor{blue}{(}\textcolor{blue}{GA}\textcolor{black}{O}\textcolor{blue}{).}
    
    \textcolor{black}{The} \textcolor{blue}{report} \textcolor{black}{highlight}\textcolor{blue}{s} \textcolor{blue}{the} \textcolor{black}{key} \textcolor{blue}{issues} \textcolor{black}{and} \textcolor{black}{deb}\textcolor{blue}{ates} \textcolor{blue}{surrounding} \textcolor{black}{D}\textcolor{blue}{HS} \textcolor{blue}{appropri}\textcolor{blue}{ations}\textcolor{blue}{,} \textcolor{blue}{including} \textcolor{black}{fund}\textcolor{blue}{ing} \textcolor{blue}{for} \textcolor{blue}{border} \textcolor{blue}{security}\textcolor{blue}{,} \textcolor{black}{imm}\textcolor{blue}{igration} \textcolor{blue}{enfor}\textcolor{blue}{cement}\textcolor{blue}{,} \textcolor{black}{cy}\textcolor{blue}{ber}\textcolor{blue}{security}\textcolor{blue}{,} \textcolor{blue}{and} \textcolor{blue}{dis}\textcolor{black}{aster} \textcolor{blue}{response}\textcolor{blue}{.} \textcolor{blue}{It} \textcolor{blue}{also} \textcolor{black}{discuss}\textcolor{blue}{es} \textcolor{blue}{the} \textcolor{blue}{impact} \textcolor{blue}{of} \textcolor{blue}{the} \textcolor{black}{B}\textcolor{blue}{CA} \textcolor{blue}{on} \textcolor{black}{D}\textcolor{blue}{HS} \textcolor{blue}{appropri}\textcolor{black}{ations} \textcolor{blue}{and} \textcolor{blue}{the} \textcolor{black}{potential} \textcolor{blue}{for} \textcolor{blue}{future} \textcolor{black}{changes} \textcolor{blue}{to} \textcolor{blue}{the} \textcolor{black}{sp}\textcolor{blue}{ending} \textcolor{blue}{caps}\textcolor{blue}{.}
    
    \textcolor{blue}{Over}\textcolor{black}{all}\textcolor{blue}{,} \textcolor{blue}{the} \textcolor{blue}{report} \textcolor{blue}{provides} \textcolor{blue}{a} \textcolor{black}{compreh}\textcolor{blue}{ensive} \textcolor{blue}{analysis} \textcolor{blue}{of} \textcolor{blue}{the} \textcolor{black}{annual} \textcolor{blue}{appropri}\textcolor{blue}{ations} \textcolor{blue}{for} \textcolor{blue}{D}\textcolor{blue}{HS} \textcolor{black}{and} \textcolor{blue}{the} \textcolor{black}{factors} \textcolor{blue}{that} \textcolor{blue}{influence} \textcolor{blue}{the} \textcolor{black}{allocation} \textcolor{blue}{of} \textcolor{blue}{fund}\textcolor{blue}{ing}\textcolor{blue}{.} \textcolor{blue}{It} \textcolor{black}{is} \textcolor{blue}{a} \textcolor{blue}{valuable} \textcolor{blue}{resource} \textcolor{blue}{for} \textcolor{blue}{polic}\textcolor{black}{ym}\textcolor{blue}{akers}\textcolor{blue}{,} \textcolor{black}{anal}\textcolor{blue}{yst}\textcolor{blue}{s}\textcolor{blue}{,} \textcolor{blue}{and} \textcolor{blue}{st}\textcolor{black}{ake}\textcolor{blue}{hold}\textcolor{blue}{ers} \textcolor{blue}{interested} \textcolor{blue}{in} \textcolor{blue}{understanding} \textcolor{black}{the} \textcolor{blue}{complex}\textcolor{blue}{ities} \textcolor{blue}{of} \textcolor{blue}{D}\textcolor{blue}{HS} \textcolor{black}{appropri}\textcolor{blue}{ations} \textcolor{blue}{and} \textcolor{blue}{the} \textcolor{black}{challeng}\textcolor{blue}{es} \textcolor{blue}{facing} \textcolor{blue}{the} \textcolor{black}{department} \textcolor{blue}{in} \textcolor{blue}{the} \textcolor{black}{coming} \textcolor{blue}{years}\textcolor{blue}{.}

\end{tcolorbox}


\end{document}